\documentclass[conference]{IEEEtran}
\IEEEoverridecommandlockouts
\usepackage{cite}
\usepackage{amsmath,amssymb,amsfonts}
\usepackage{graphicx}
\usepackage{textcomp}
\usepackage{xcolor}
\usepackage[
            colorlinks,
            linkcolor=red, 
            anchorcolor=green,
            citecolor=blue,
            ]{hyperref}
\def\BibTeX{{\rm B\kern-.05em{\sc i\kern-.025em b}\kern-.08em
    T\kern-.1667em\lower.7ex\hbox{E}\kern-.125emX}}

\definecolor{lightblue}{HTML}{D1EDFF}
\definecolor{red_one}{rgb}{0.8392156862745098, 0.15294117647058825, 0.1568627450980392}
\definecolor{blue_one}{rgb}{0.12156862745098039, 0.4666666666666667, 0.7058823529411765}
\usepackage[font=small,labelfont=bf,tableposition=top]{caption}
\DeclareCaptionLabelFormat{andtable}{#1~#2  \&  \tablename~\thetable}
\definecolor{-}{rgb}{0.25,0.41,0.88}
\definecolor{+}{rgb}{0.70,0.13,0.13}

\newcommand{\Review}[2]{\textbf{Comment #1: } \textit{\textcolor{blue!50!black}{#2}}}
\newcommand{\Response}[2]{\textbf{Response #1: }#2 \vspace{0.5em}}

\usepackage{amssymb,amsmath,amsthm}
\usepackage{graphicx}

\usepackage{multirow}
\usepackage{mathtools}
\usepackage{colortbl}

\usepackage{algorithm}
\usepackage{listings}
\usepackage{algpseudocode} 
\usepackage{enumitem}
\newtheorem{corollary}{Corollary}[section]
\usepackage{booktabs}
\usepackage{natbib}

\usepackage{tikz}
\usepackage{pgfplots}
\usepackage{verbatim}
\newcommand{\ie}{\emph{i.e., }}
\newcommand{\eg}{\emph{e.g., }}

\newcommand{\wrt}{\emph{w.r.t. }}
\newcommand{\cf}{\emph{cf. }}
\newcommand{\aka}{\emph{aka. }}

\setcitestyle{numbers,square}
\begin{document}
\title{BSL: Understanding and Improving Softmax Loss for Recommendation
}
\author{
  \IEEEauthorblockN{
    Junkang Wu\textsuperscript{1},
    Jiawei Chen\textsuperscript{2}\IEEEauthorrefmark{1},
    Jiancan Wu\textsuperscript{1}\IEEEauthorrefmark{1},
    Wentao Shi\textsuperscript{1},
    Jizhi Zhang\textsuperscript{1},
    Xiang Wang\textsuperscript{1}\IEEEauthorrefmark{2}
  }
  \IEEEauthorblockA{\textsuperscript{1} \itshape MoE Key Laboratory of Brain-inspired Intelligent Perception and Cognition, \\
  University of Science and Technology of China, Hefei, China}
  \IEEEauthorblockA{\textsuperscript{2} \itshape The State Key Laboratory of Blockchain and Data Security, Zhejiang Unversity}
  \IEEEauthorblockA{ \ttfamily
    \{jkwu0909, wujcan, xiangwang1223\}@gmail.com,
    sleepyhunt@zju.edu.cn,\\
    \{shiwentao123, cdzhangjizhi\}@mail.ustc.edu.cn
  }
  \thanks{\IEEEauthorrefmark{1}Jiawei Chen and Jiancan Wu are the corresponding authors.}
  \thanks{\IEEEauthorrefmark{2}Xiang Wang is also affiliated with Institute of Artificial Intelligence, Institute of Dataspace, Hefei Comprehensive National Science Center.}
}

\maketitle

\begin{abstract}
Loss functions steer the optimization direction of recommendation models and are critical to model performance, but have received relatively little attention in recent recommendation research. Among various losses, we find Softmax loss (SL) stands out for not only achieving remarkable accuracy but also better robustness and fairness. Nevertheless, the current literature lacks a comprehensive explanation for the efficacy of SL.

Toward addressing this research gap, we conduct theoretical analyses on SL and uncover three insights: 1) Optimizing SL is equivalent to performing Distributionally Robust Optimization (DRO) on the negative data, thereby learning against perturbations on the negative distribution and yielding robustness to noisy negatives. 2) Comparing with other loss functions, SL implicitly penalizes the prediction variance, resulting in a smaller gap between predicted values and and thus producing fairer results. 

Building on these insights, we further propose a novel loss function Bilateral SoftMax Loss (BSL) that extends the advantage of SL to both positive and negative sides. BSL augments SL by applying the same Log-Expectation-Exp structure to positive examples as is used for negatives, making the model robust to the noisy positives as well. Remarkably, BSL is simple and easy-to-implement --- requiring just one additional line of code compared to SL. Experiments on four real-world datasets and three representative backbones demonstrate the effectiveness of our proposal. The code is available at \href{https://github.com/junkangwu/BSL}{https://github.com/junkangwu/BSL}. 
\end{abstract}

\begin{IEEEkeywords}
  Recommendation System, Robustness, Distributionally Robust Optimization
\end{IEEEkeywords}

\section{Introduction}
Recommender Systems (RS) have become a leading technology addressing information overload.
Owing to the easy availability of
user behaviors (\eg click and purchase), collaborative filtering (CF) with implicit feedback emerges as an effective solution. Recent years have witnessed the flourishing publications on designing model architecture, spanning matrix factorization \cite{pan2008one},
autoencoders\cite{liang2018variational, ma2019learning, steck2019embarrassingly}, and advanced graph neural networks \cite{he2020lightgcn, ying2018graph,wang2019neural}. In contrast, relatively less work focuses on loss functions, which play critical roles in steering CF.

Existing loss functions for CF with implicit feedback can be mainly categorized into three types: 1) Pointwise loss (\eg BCE \cite{he2017neural_a}, MSE \cite{he2017neural}) casts the problem into a classification or regression, promoting model predictions close to the labels; 2) Pairwise loss (\eg BPR \cite{rendle2012bpr}) encourages a higher score for positive item compared to its negative counterpart;
3) Softmax loss (SL) \cite{wu2022effectiveness, bengio2003quick} normalizes model predictions over items into a multinomial distribution, optimizing the probability of positive instances versus negative ones. 

    \begin{figure}[t]\centering
    \includegraphics[width=1.0\linewidth]{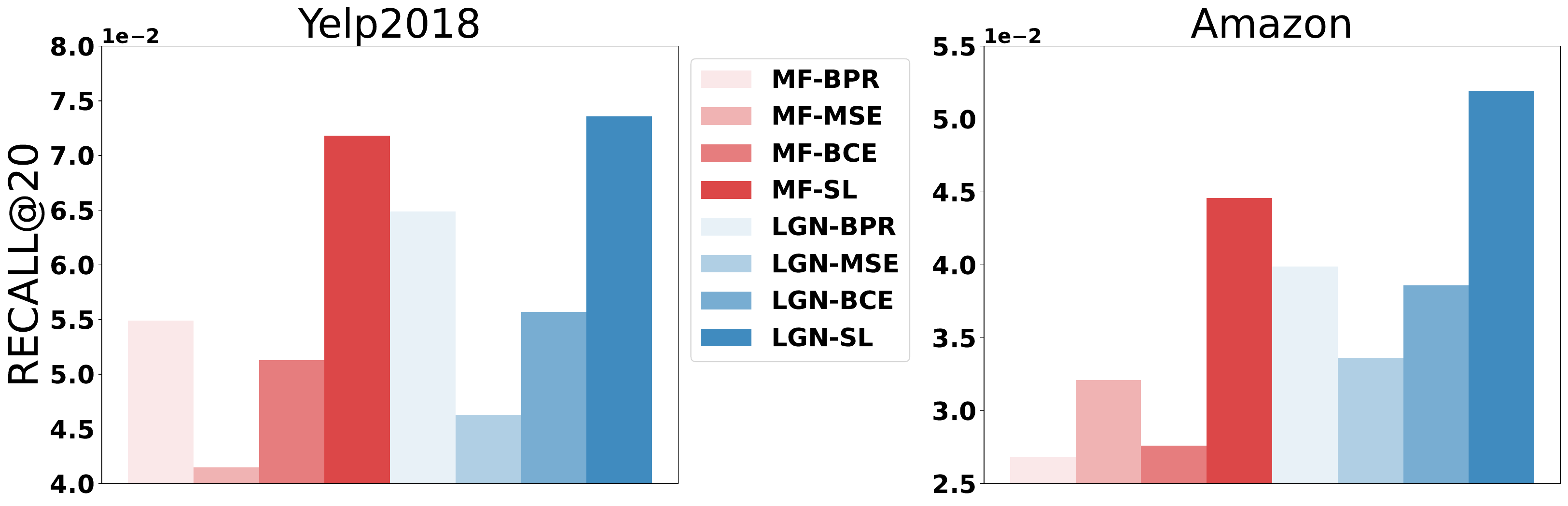}
    \caption{ 
    Performance comparison of different loss functions based on MF (\textcolor{red_one}{red bars}) and LightGCN (\textcolor{blue_one}{blue bars}) on the Yelp2018 and Amazon datasets. SL consistently achieves superior performance over other losses by a significant margin ($>15\%$ gain) across models and datasets. Additional results on other datasets and architectures (presented in Table I) exhibit similar patterns, highlighting the effectiveness of SL for collaborative filtering.}
    \label{fig::t_changes_performance_yelp2}
  \end{figure}

Through empirical evaluation of several losses paired with two representative recommenders (\ie- MF\cite{MF} and LightGCN\cite{he2020lightgcn}) on real-world datasets, we find SL consistently and significantly outperforms other losses, with over 15\% higher accuracy, as is shown in Figure \ref{fig::t_changes_performance_yelp2}.
Beyond accuracy, we find SL also exhibits better performance in terms of robustness and fairness (\cf, Figure \ref{fig::t_changes_performance}).
This impressive result reveals the superiority of SL, which motivates us to explore a fundamental question --- \textbf{what are the underlying reasons behind SL's effectiveness?} 
Prior work \cite{covington2016deep, bruch2019analysis, wu2022effectiveness} has explored the properties of SL in recommendations but with several limitations.
  1) The connections between SL and other losses, as well as how SL improves upon them, remain unrevealed;  2) The fairness of SL was attributed to the popularity-based negative sampling strategy, rather than the loss itself, whereas we find uniform sampling also yields fairness.


In this work, we strive to conduct in-depth analysis of SL through the lens of \textit{Distributionally Robust Optimization} (DRO) \cite{ben2013robust, duchi2016statistics, hu2013kullback}. DRO is an advanced stochastic optimization framework that learns models robust to perturbations in the data distribution. Remarkably, our rigorous theoretical analyses reveal the following key insights:

\begin{itemize}[leftmargin=*]
\item \textbf{Optimizing SL is equivalent to performing DRO on the vanilla pointwise loss}, where certain perturbations are added to the negative item distribution.
{In fact, noisy (\ie false) negative instances are common in RS due to factors, such as user unawareness, low ranking position or popularity of items~\cite{zhuang2021cross,canamares2018should}.
Such noisy negative data inevitably causes distribution shift between sampled data and the true one, misleading learning procedure and deteriorating performance.
In contrast, the DRO nature of SL imparts robustness to negative noise, and hence yielding better performance.}
\item \textbf{SL introduces an implicit regularizer that controls the prediction variance on negative instances.} 
Due to the long tail distribution exhibited in recommendation data,
models tend to overly favor popular items while neglecting others, resulting in notorious popularity bias. By penalizing the variance of model predictions, SL would reduce the prediction discrepancy between popular and unpopular items, leading to fairer recommendation results.

 \end{itemize}

The above analyses not only explain the strengths of SL, but inspire us an enhanced loss design. 
Intuitively, robustness to only negative noise is insufficient, as positive instances also contain unreliability due to the factors like clickbait \cite{Jagerman2019To} or user conformity \cite{zhang2021causal}.
In light of this, we propose a novel loss function named \textit{\textbf{B}ilateral \textbf{S}oftmax \textbf{L}oss} (BSL), which extends the advantage of SL to both positive and negative sides. Specifically,
we augment the positive loss of SL to a Log-Expectation-Exp structure mirroring the negative part.
This bilateral architecture provides robustness to noise on both positives and negatives.
Remarkably, our BSL is easy to implement, requiring only one line of code amended as compared with SL. 


\textbf{Contributions.} In summary, our contributions are:
\begin{itemize}[leftmargin=*]
    \item We prove that optimizing SL is equivalent to performing DRO, which provides novel theoretical insights into SL's strengths {on both robustness and fairness}. 
    \item We propose a novel Bilateral Softmax Loss, which is easily implemented while achieving bilateral robustness to both positive and negative noise.
    \item We conduct extensive experiments on four real-world datasets and three representative backbones to demonstrate the effectiveness of BSL.
\end{itemize}

The remainder of this paper is structured as follows. 
Section \ref{sec_2} provides background on recommender systems and existing loss functions, followed by a brief introduction on distributionally robust optimization.
Our analysis on SL through theoretical proofs and empirical experiments is presented in Section \ref{sec_3}. 
Section \ref{sec_4} identifies SL's weaknesses and proposes our positive denoising loss --- Bilateral Softmax Loss.
The effectiveness of our proposal is evaluated through extensive experiments in Section \ref{exp_main}. Finally, 
Section \ref{sec_6} and \ref{sec_7} review related work on robust recommendation and draw conclusion.
\section{Preliminaries}
\label{sec_2}
In this section, we present the background of Recommender System and Distributionally Robust Optimization. 
\subsection{Recommender System}


Suppose we have a recommender system (RS) with a user set $\mathcal U$ and an item set $\mathcal I$. 
The historical user-item interactions can be expressed by a matrix $\mathbf R \in\{0,1\}^{|\mathcal U| \times |\mathcal I|}$, whose element $r_{ui}$ represents whether a user $u$ has interacted (\eg click) with an item $i$. For convenience, we define the positive (or negative) items for user $u$ as $\mathcal S^+_u\equiv\{i\in\mathcal{I}|r_{ui}=1\}$ (or $\mathcal S^-_u\equiv\{i\in\mathcal{I}|r_{ui}=0\}$). The goal of a RS is to recommend new items for each user that he may be interested in.


Regarding model optimization, existing work typically frames the task as stardard supervised learning.
Let ${P}^+_u$ and ${P}^-_u$ denote the distribution of positive and negative training instances respectively. 
${P}^+_u$, is usually set to a uniform distribution over $\mathcal{S}^+_u$, while ${P}^-_u$ can be configured either as a uniform distribution over $\mathcal{S}^-_u$ or using population-based sampling from $\mathcal{S}^-_u$.
There are mainly three types of loss functions in RS \cite{rendle2022item}:
\begin{itemize}[leftmargin=*]
    \item Pointwise loss frames the problem of recommendation from implicit feedback as a classification or regression problem, encouraging model predictions close to the corresponding labels:
  
      \begin{equation}
        \begin{aligned}
          \mathcal{L}_{Pointwise}(u)= & -\mathbb E_{i \sim {P}^+_u} [\phi (f(u,i))] \\
          & +  c\mathbb E_{j\sim {P}^-_u} [\psi (f(u,j))]
        \end{aligned}
      \end{equation}

      where $f(u,i)$ is the predicted score, $c$ balances positive and negative contributions, and $\phi(\cdot)$ and $\psi(\cdot)$ are functions adapted for different loss choices. For example, in terms of BCE and MSE, we have:
\begin{equation}
\begin{aligned}
&\left\{ {\begin{aligned}
\phi {{(f(u,i))}_{BCE}} &=  \log \sigma (f(u,i))\\
\psi {{(f(u,j))}_{BCE}} &=  - \log \left( {1 - \sigma (f(u,j))} \right)
\end{aligned}} \right. \\
&\left\{ {\begin{aligned}
\phi(f(u,i))_{MSE}&= -{||f(u,i) - 1||}^2\\
\psi(f(u,j))_{MSE}&= {||f(u,j) - 0||}^2 
\end{aligned}} \right.
\end{aligned}
\end{equation}


    \item Pairwise loss enforces higher scores for positive items versus the negative counterparts:
    \begin{gather}
      \mathcal{L}_{Pairwise}(u)= - \mathbb E_{i \sim P^+_u,\  j\sim P^-_u} \left[\varphi\left(f(u,i) - f(u,j)\right)\right]
    \end{gather}
    For the classical BPR loss, $\varphi(.)$ can is set to log-Sigmoid function.
    \item Softmax Loss
    normalizes model predictions into a multinomial distribution and optimizes the probability of positive instances over negative ones\footnote{Note that here we simply remove the positive term in the denominator as it makes negligible contribution when $N^-$ takes a relatively large value. Also, recent work suggests that this treatment could boost the embedding uniformity and empirically yield slightly better performance \cite{yeh2022decoupled}.}:
        \begin{gather}
          \label{sl_eq}
          \mathcal{L}_{SL} (u)= - \mathbb E_{i \sim P^+_u} \Bigg [\log 
{\frac{\exp{(f(u,i)/\tau)}}{N^-\mathbb E_{ j\sim P^-_u} [\exp{(f(u,j)/\tau)}]}} \Bigg ]
        \end{gather}
        where $N^-$ denotes the number of sampled negative instances. 
        Note that SL (\cf Eq.(\ref{sl_eq})) can be expressed by a combination of positive and negative parts:
        \begin{equation}
          \label{eq_SL}
        \begin{aligned}
          \mathcal{L}_{SL}(u) &= \underbrace{-\mathbb{E}_{i \sim P^+_u} [f(u,i)]}_{\text{Positive Part}} + \underbrace{{\tau}\log \mathbb E_{j\sim P^-_u}[\exp(f(u,j)/\tau)]}_{\text{Negative Part}}
          \end{aligned}
        \end{equation}
        Upon comparison between the $\mathcal{L}_{SL}$ and the point-wise loss, we discern a significant distinction primarily within the negative part --- SL utilizes a specific Log-Expectation-Exp structure (\ie $\log\mathbb E[\exp(.)]$) in the negative part.
        In the forthcoming section, we will elucidate the advantages inherent to this structure. Notably, an additional hyperparmeter $\tau$ is always introduced in SL. The role of this hyperparameter within the context of SL will also be explored.
\end{itemize}
\subsection{Distributionally Robust Optimization}

The success of machine learning relies on the assumption that the test distribution matches the training distribution (\aka, iid assumption). However, this assumption fails to hold in many real-world applications, leading to sub-optimal performance. Distributionally Robust Optimization (DRO) \cite{ben2013robust, duchi2016statistics, hu2013kullback} 
addresses this by demanding models to perform well not only on the observed data distribution, but over a range of distribution with perturbations.
Formally, DRO plays a min-max game: it first identifies the most difficult distribution (\aka worst-case distribution) from the uncertainty set $\mathbb{P}$, and then optimizes the model under that distribution: 
\begin{equation}
  \label{eq2}
  \begin{aligned}
      \hat{\theta}&= \operatorname*{argmin}_\theta \{ \operatorname*{max}_{P \in \mathbb{P} }  \mathbb{E}_{x\sim P}[\mathcal{L}(x;\theta)] \} \\
      \mathbb{P}&=\{P\in \mathbb D:D(P,P_o)\leq \eta\}
      \end{aligned}
\end{equation}


\noindent where $\mathbb D$ denotes the set of all distributions, $D(.,.)$ measures the distance between the original distribution $P_o$ and the perturbed distribution $P$; $\eta$ denotes the distance threshold. The uncertainty set $\mathbb{P}$ includes the possible distributions to optimize well over. Hence, the choice of $D$ and $\eta$ are critical as they control the uncertainty set $\mathbb{P}$.

\section{Understanding SL from DRO}
\label{sec_3}
In this section, we conduct thorough analyses on SL with both theoretical proofs and empirical experiments.
\subsection{Theoretical Analyses}

\newtheorem{lem}{Lemma}
\begin{lem}
  \label{lam0}
  Optimizing Softmax loss (\cf Equation \ref{sl_eq}) is equivalent to performing Distributionally Robust Optimization over the original point-wise loss, \ie optimizing: 

  \begin{equation}\label{l_dro}
  \begin{aligned}
    \mathcal{L}_{DRO}(u) &= -\mathbb{E}_{i \sim P^+_u} [f(u,i)] + \operatorname*{max}_{P \in \mathbb{P} } \mathbb E_{j\sim P}[f(u,j)] \\
    \mathbb{P}&=\{P\in \mathbb D:D_{KL}(P,P^-_u)\leq \eta \}
    \end{aligned}
  \end{equation}
  where $\mathbb{P}$ denotes the uncertainty set of the negative data distribution; $D_{KL}$ denotes the KL-divergence from distribution $P$ to the original $P^-_u$; $\eta$ denotes the robustness radius controlling how $P$ deviates from $P^-_u$.
\end{lem}

\begin{proof}

By comparing the Softmax loss (Equation ($\ref{sl_eq}$)) with $\mathcal{L}_{DRO}(u)$ (Equation ($\ref{l_dro}$)), we can find the difference mainly lies on the negative part. In fact, the equivalence of the Log-Expectation-Exp structure and KL-based DRO objective has been revealed by recent work on other areas \cite{hu2013kullback} \cite{wu2023ADNCE}. To ensure the completeness of the article, here we also give the complete proof by referring to these work. 

Let $L(j)=P(j)/P^-_u(j)$. Note that the KL-divergence between $P$ and $P^-_u(j)$ is constrained, and thus $L(.)$ is fine definition. For brevity, we usually short $L(j)$ as $L$. Also, we define a specific convex function $g(x)=x\log x$. Then the KL divergence $D_{KL}(P,P^-_u)$ can be written as $\mathbb{E}_{j\sim P^-_u} [g(L)]$. 
Equation \eqref{l_dro} can be reformulated as follow:

\begin{equation}
  \begin{aligned}
    \label{lam0_new}
    &- \mathbb{E}_{i \sim P^+_u} [f(u,i)] + \operatorname*{max}_{L} \mathbb{E}_{j\sim P^-_u} [f(u,j)L]\\
    &\text{subject to } \mathbb{E}_{j\sim P^-_u} [g(L)] \leq \eta, \quad \mathbb{E}_{P^-_u} [L] = 1\\
\end{aligned}
\end{equation}

Considering the convex of $\mathbb{E}_{P^-_u} g(L)$ and $\mathbb{E}_{P^-_u} [f(u,j)L]$ on $L$, 
implying it is a convex optimization problem, we use the Lagrangian function to solve it:
\begin{equation}
  \begin{aligned}
    \label{lam0_Lagrangian}
    -\mathbb{E}_{i \sim P^+_u} [f(u,i)] + &\min_{\alpha \ge 0,\beta}\operatorname*{max}_{L}  \{  \mathbb{E}_{j\sim P^-_u} [f(u,j)L -  \\ 
    &\alpha (\mathbb{E}_{j\sim P^-_u} [g(L)] - \eta) + \beta (\mathbb{E}_{j\sim P^-_u} [L] - 1) \}
\end{aligned}
\end{equation}
where $\alpha$ and $\beta$ are Lagrangian multipliers. Now we focus on eliminate $L$ and $\beta$. In fact, we have: 
\begin{equation}
    \begin{aligned}
      &\operatorname*{min}_{\beta} \operatorname*{max}_{L} \{ \mathbb{E}_{j\sim P^-_u} [f(u,j)L -  \\
       &\alpha (\mathbb{E}_{j\sim P^-_u} [\phi(L)] - \eta) + \beta (\mathbb{E}_{P^-_u} [L] - 1) \} \\
      =& \operatorname*{min}_{\beta} \big \{ \alpha \eta - \beta + \alpha \operatorname*{max}_{L} \{\mathbb{E}_{ P^-_u}[\frac{f(u,j) + \beta }{\alpha}L] - \mathbb{E}_{P^-_u}[g(L)] \} \big \} \\
      =&\operatorname*{min}_{\beta} \big \{ \alpha \eta - \beta + \alpha \mathbb{E}_{ P^-_u} [ \operatorname*{max}_{L} \{\frac{f(u,j) + \beta }{\alpha}L - g(L) \} ] \big \} \\
      =&\operatorname*{min}_{\beta} \big \{ \alpha \eta - \beta + \alpha \mathbb{E}_{ P^-_u} [ g^*(\frac{f(u,j) + \beta }{\alpha}) ] \big \}
    \end{aligned}
\end{equation}
The last equality held since the definition of convex conjugate $g^*(y)=\operatorname*{max}_{x}\{ xy - g(x) \}$. Given $g(x)=x\log x$, we have $g^*(x)=e^{x-1}$:
\begin{equation}\label{key_eq}
    \begin{aligned}
      &\operatorname*{min}_{\beta} \big \{ \alpha \eta - \beta + \alpha \mathbb{E}_{ P^-_u} [ g^*(\frac{f(u,j) + \beta }{\alpha}) ] \big \} \\
      =& \operatorname*{min}_{\beta} \big \{ \alpha \eta - \beta + \alpha \mathbb{E}_{ P^-_u} [ e^{\frac{f(u,j) + \beta }{\alpha}-1} ] \big \} \\
      =& C + \alpha \log \mathbb{E}_{ P^-_u} [e^{\frac{f(u,j) }{\alpha}}]
    \end{aligned}
\end{equation}
where $C=\alpha \eta$ denotes a constant. The Last equality holds due to the optimal solution found for $\beta$.
By plugging Eq.~\eqref{key_eq} into problem~\eqref{lam0_Lagrangian}, we get the overall optimal result as follows:
\begin{equation}
  \begin{aligned}
    \label{lam0_finnal}
    &-\mathbb{E}_{ P^+_u} [f(u,i)] +\min_{\alpha \ge 0,\beta}\operatorname*{max}_{L} \bigl\{  \mathbb{E}_{ P^-_u} [f(u,j)L - \\
    &  \hspace{2cm} (\mathbb{E}_{ P^-_u} [\phi(L)] - \eta)  + \beta (\mathbb{E}_{P^-_u} [L] - 1) \bigr\}  \\
    =&\min_{\alpha \ge 0}\bigl\{- \mathbb{E}_{ P^+_u} [f(u,i)] + C + \alpha \log \mathbb{E}_{ P^-_u} [\exp({{f(u,j) }/{\alpha}})]\bigr\}  \\
    =&- \mathbb{E}_{ P^+_u} [f(u,i)]  + \alpha^{*} \log \mathbb{E}_{ P^-_u} [\exp({{f(u,j) }/{\alpha^{*}}})] + C  \\
  \end{aligned}
\end{equation}
Where the optimal $\alpha^*$ is determined by the hyperparameter $\eta$. We may reversely treat $\alpha^*$ as a surrogate hyperparameter and consider $\eta$ as a function of $\alpha^*$. In this way, by comparing Eq.\eqref{lam0_finnal} with SL and removing constant terms that is irrelevant with $f(u,i)$, we can conclude that optimizing SL is equivalence to optimizing $\mathcal{L}_{DRO}(u)$ (Eq.\eqref{l_dro}) and $\alpha$ is identical to $\tau$.    

\end{proof}
{\textbf{Remark 1 (Capability of Robustness)}} This lemma points out the essence and superiority of SL compared to other loss functions. In RS, the presence of noise inevitably leads to a significant gap between the distribution of sampled items and the true negative distribution (\aka distributional shifts). Due to the intrinsic connection to DRO, SL explicitly considers the  uncertainty of distributions to seek a model that performs well against distribution perturbation. Specifically, if the uncertainty set $\mathcal{P}$ is appropriately selected and the true negative data distribution is contained in the built uncertainty set, the robustness of the model could be guaranteed under distributional shifts, and performance superiority would begin to emerge.

\begin{lem}[Fairness]
  \label{lam1}
  With the Lemma \ref{lam0}, we can approximate the Softmax loss as follow:

  \begin{equation}
    \begin{aligned}
      \label{lam1_eq}
       \mathcal{L}_{SL} (u)=& -\mathbb{E}_{i \sim P^+_u} [f(u,i)] + \operatorname*{max}_{P \in \mathbb{P} } \mathbb E_{j\sim P}[f(u,j)] \\
    \approx&-\mathbb{E}_{i \sim P^+_u} [f(u,i)]+\mathbb{E}_{j\sim P^-_u} [f(u,j)] \\ 
    + & \frac{\mathbb V[f(u,j)]}{2\tau}+\tau \eta+o_{\infty}(1/\tau) 
  \end{aligned}
  \end{equation}

  where $\mathbb V[f{(u,j)}]$ denotes the variance of score $f({u,j})$ under the distribution $P^-_u$.
\end{lem}
\begin{proof}
  As the optimal value of $\operatorname*{max}_{P\in \mathbb{P}} \mathbb{E}_{j\sim P} [f(u,j)]$ can be written as $\tau \log \mathbb{E}_{j\sim P^-_u} [e^{f(u,j)/\tau}] + \tau \eta $. To ease notations, we denote $\Omega (\gamma )=\log \mathbb{E}_{j\sim P^-_u}[ e^{f(u,j)\gamma}]$ and $\gamma=1/\tau$.
  By leveraging a second-order Taylor expansion around 0 on $\Omega (\gamma )$, we have:
  \begin{equation}
    \label{lam1_eq2}
    \Omega (\gamma ) = \gamma \mathbb{E}_{j\sim P^-_u} [f(u,j)] + \frac{{\gamma}^2}{2}\mathbb V[f(u,j)] + o_0 ({\gamma}^2)
  \end{equation}
  Thereafter, we plug $\tau=1/\gamma$ into $\Omega (\gamma )$, which finishes the proof:
  \begin{equation}
    \label{lam1_eq3}
    \begin{aligned}
      &\operatorname*{max}_{P \in \mathbb{P} } \mathbb E_{j\sim P}[f(u,j)]=\tau \log \mathbb{E}_{j\sim P^-_u} [e^{f(u,j)/\tau}] + \tau \eta \\
      & \approx \mathbb{E}_{j\sim P^-_u} [f(u,j)] + \frac{\mathbb V[f(u,j)]}{2\tau} + o_{\infty}(1/\tau) + \tau \eta 
    \end{aligned}
  \end{equation}
\end{proof}
{\textbf{Remark 2 (Capability of Fairness)}} {
This lemma provides the insight that SL pursues the fairness of the recommendation system through regularization on the variance of negative samples $f(u,j)$. 
Recent research \cite{williamson2019fairness} has indicated that the variance penalty corresponds to fairness. From Equation \eqref{lam1_eq}, we can conclude that SL implicitly introduces a regularizer that penalizes the variance of the model predictions on the negative instances. It is important to note that in a typical recommender systems, recommendation model is prone to give extensive higher scores on popular items than unpopular ones, incurring popularity bias \cite{chen2021autodebias}. SL could mitigate this effect to a certain degree, as it inherently reduces the variance of model predictions, which implies that the predictions are likely to be more uniform. Consequently, the discrepancy in model predictions between popular and unpopular items is diminished, leading to more fair recommendation results. Unpopular items get more opportunity in recommendation, and thus we observe better NDCG@20 for those unpopular groups.}

\begin{corollary}[The optimal $\alpha^*$ - Lemma 5 of \cite{faury2020distributionally}]
  \label{coro1}
  The value of the optimal $\alpha^*$ (i.e., temperature $\tau$) can be approximated as follows:
    \begin{align}
      \label{eta_}
        \tau^* \approx \sqrt{\frac{\mathbb V[f{(u,j)}]}{2\eta}}.
    \end{align}
    where $\mathbb V[f{(u,j)}]$ denotes the variance of score $f({u,j})$ under the distribution $P^-_u$.
\end{corollary} 



{\textbf{Remark 3 (Role of the Temperature $\tau$).}} From Lemma 1, we can see that $\tau$ is not just a heuristic design, but can be interpreted as a Lagrange multiplier in solving DRO optimization problem. In additional, Corollary \ref{coro1} indicates that $\tau$ is determined by the $\eta$ and the variance of $f{(u,j)}$. It means that tuning the hyperparameter $\tau$ in SL is essentially equivalent to adjusting the robustness radius. 

The above insight also could help us to better understand how temperature affects model performance. When the value of $\tau$ is excessive large, it creates a limited uncertainty set which fails to encompass the true negative distribution, resulting in a lack of robustness (\cf Figure \ref{fig::ppt_Performance} left). As $\tau$ decreases, the robustness radius $\eta$ grows, the model is exposed to more challenging distributions. 
Given that the uncertainty set is more prone to encompass the true distribution, there is a potential to enhance both the robustness and accuracy of the model. 
However, a continual decrease of $\tau$ would also increase the risk of implausible worst distribution. Typically, when $\eta$ is too large, the worst-case distribution in DRO may not be the desired distribution but rather uninformative one hurting model performance (\cf Figure \ref{fig::ppt_Performance} right). These two aspects create trade-off of selecting $\tau$. we will empirical verify it in subsection \ref{tau_exp}.

\begin{figure}[t]\centering
  \includegraphics[width=1.0\linewidth]{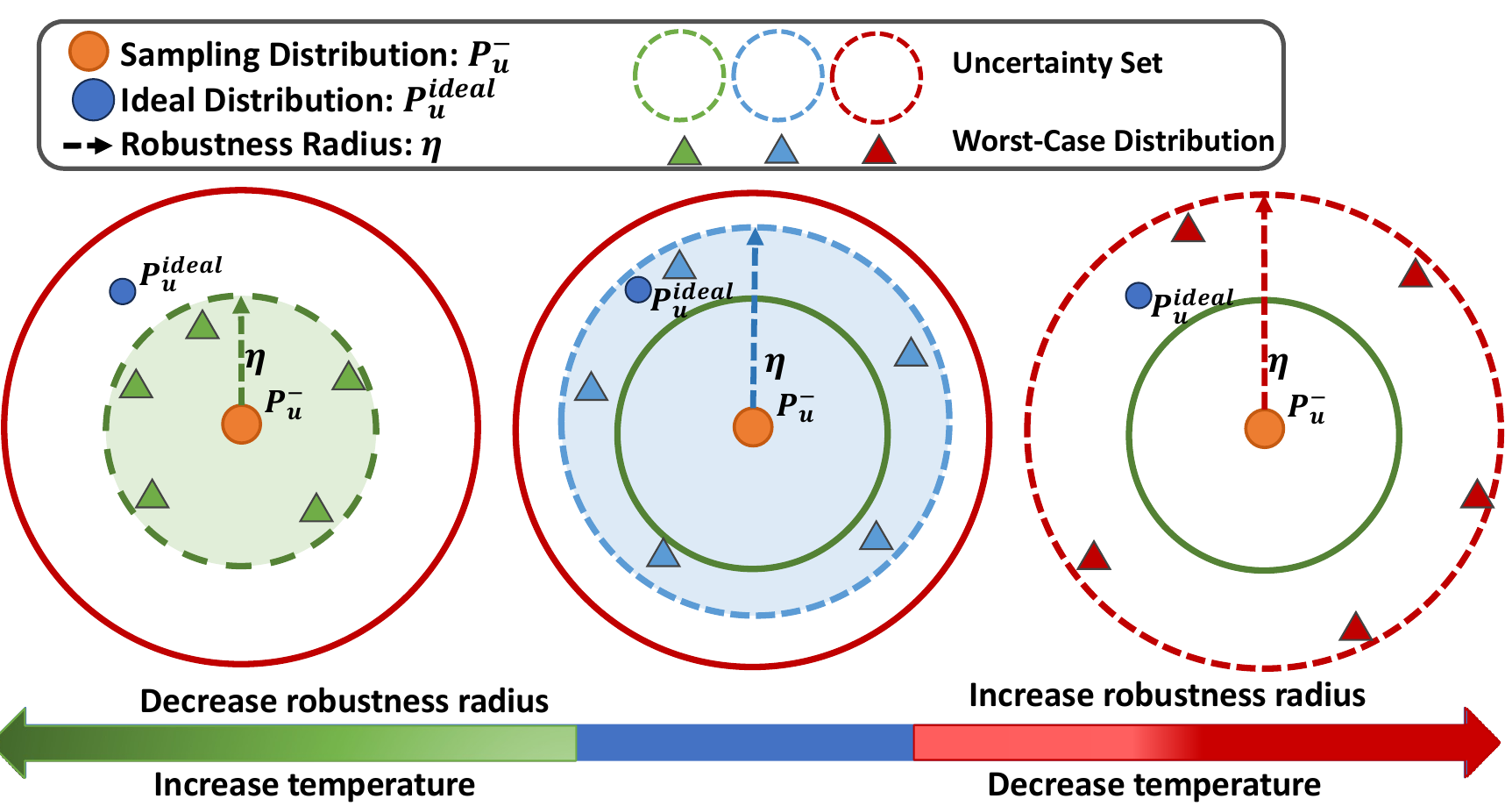}
  \caption{{The relationships between $\tau$, $\eta$ and uncertainty set}}
  \label{fig::ppt_Performance}
\end{figure}

\subsection{Empirical Analysis}
In this section, we conduct the following empirical experiments to understand the essence of SL. The detailed experimental setup refers to section \ref{exp_main}. Here We take Yelp2018 dataset~\cite{he2020lightgcn} as an example, and could draw similar conclusions from other datasets.
\label{Empirical}
\subsubsection{\textbf{Experimental design.}} To understand the above theories, here we conduct four experiments. \textbf{(1)} We record the performance of SL under different choices of $\tau$ and noise level \footnote{In the sampling of negative samples, $r_{noise}$ represents the ratio of the sampling probability of positive samples to that of negative samples. Therefore, a larger value of $r_{noise}$ indicates that more positive samples are mistakenly classified as negative samples.} ($r_{noise}$). (Figure \ref{fig::noise_eta}a). The higher value of $r_{noise}$, the more noise in the negative data. It can be implemented during negative sampling, where a certain portion of positive instances are sampled as negatives. \textbf{(2)} We grid search the best $\tau$ for different scenarios and visualize the distribution of $\eta$ according to Equation \eqref{eta_} (Figure \ref{fig::noise_eta}b). \textbf{(3)} To verify fairness, we show SL's performance on different item groups with different popularity (Figure \ref{fig::t_changes_performance}a). \textbf{(4)} We also study the relationship between the prediction scores $f(u,j)$ and weights $P^*(j)$ in DRO. The weights can be understood as the sampling probability of an item in the worst-case $P^*$ (Figure \ref{fig::t_changes_performance}b)\footnote{Here $P^*(j)$ and $P^*$ both represent the worst case distribution on negative part. $P^*=\mathop{\arg\max}_{P \in \mathbb{P} } \mathbb E_{j\sim P}[f(u,j)]$ }.

\begin{figure}[t]\centering
  \includegraphics[width=1.0\linewidth]{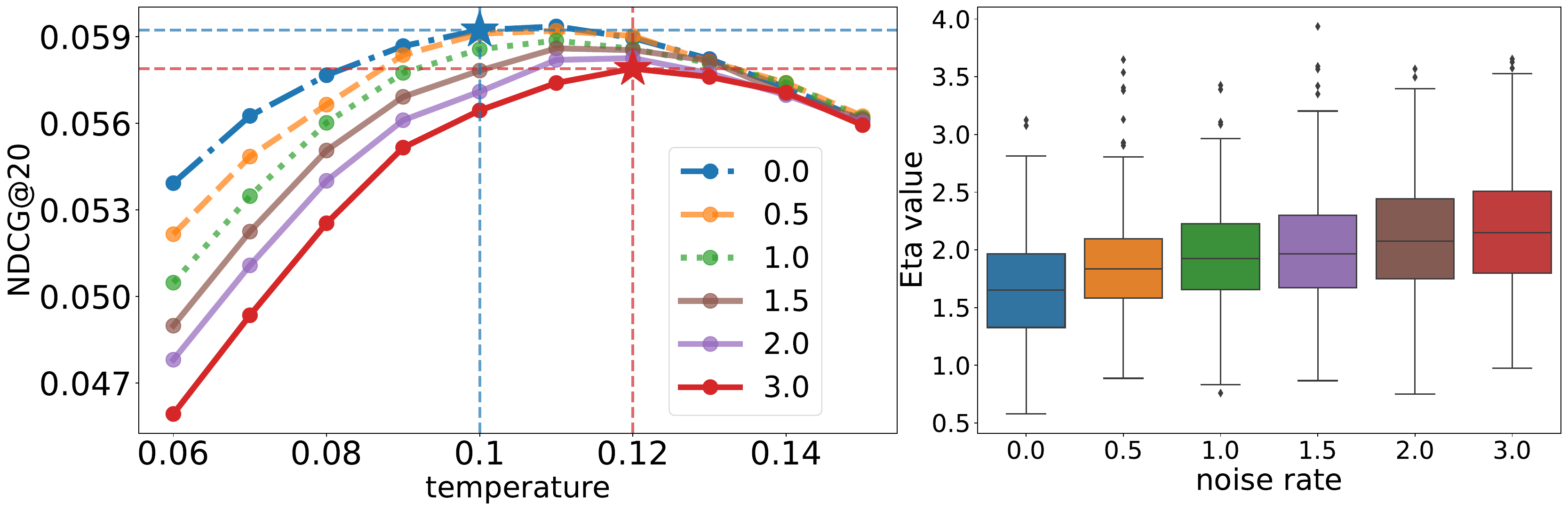}
  \caption{Robustness Analysis. (a) \textbf{Left}: we show the NDCG of SL with different noise ratios across different temperatures $\tau$. (b) \textbf{Right}: 
  we show the distribution of robustness radius $\eta$ with the best choice of $\tau$ across different noise ratios.}
  \label{fig::noise_eta}
\end{figure}

\subsubsection{\textbf{Exploring robustness \wrt $\tau$}} 
\label{tau_exp}
Figure \ref{fig::noise_eta}a shows the effect of changing the temperature on the model performance. As can be seen, the performance grows with the increase of $\tau$ at the beginning, while drops with further increase. This phenomenon is as expected. An excessively high value of $\tau$ leads to insufficient perturbations in the model adversarial training, resulting in a lack of robustness and subpar generalization performance. An overly small $\tau$ might also degrade performance due to the increasing risk of exposure to implausible and abnormal worst-case distributions. Therefore, it is important to select an appropriate value of $\tau$ to ensure optimal performance of the model.

From Figure \ref{fig::noise_eta}b, we find that when we search for the best $\tau$ at each noise level and record the corresponding distribution of $\eta$, $\eta$ actually rises with the increase of false negative samples. It is consistent with our intuition, more noisy data requires a larger robustness radius. However, as shown in Figure \ref{fig::noise_eta}a, when the noise rate ($0.0\to 3.0$) enlarges, the best $\tau$ for SL gradually increases ($0.10\to 0.12$), which is against our intuition that small $\tau$ represents larger $\eta$ to obtain strong robustness and superior performance. The contradiction is that the optimal $\tau$ is also correlated with the variance of $f{(u,i)}$. Although the robustness radius increases, the cases with more noisy samples would have a larger variance and thus incur a larger value of optimal $\tau$, which is consistent with Corollary \ref{coro1}.



\subsubsection{\textbf{Exploring worst-case distribution.}} 
To verify the advantages of SL that relieves distribution shifts in negative sample distribution, we randomly choose one batch of training data and record the negative sample score and respective worst-case $P^*$ in Equation~\eqref{l_dro}. From Figure \ref{fig::t_changes_performance}b, we observe that SL places more emphasis on {hard} negative samples. Meanwhile, if we decrease the value of $\tau$ (\eg $\tau$=0.09), the distribution of weight exhibits more "extreme", which increases the difficulty in optimization. In contrast, increasing $\tau$ (\eg $\tau$=0.11) represents a smaller robustness radius, and the distribution becomes more "gentle". We remark this phenomenon has also been observed by \cite{wu2022effectiveness}, but here we give explanation from a new perspective of DRO.


\begin{figure}[t]\centering
  \includegraphics[width=1.\linewidth]{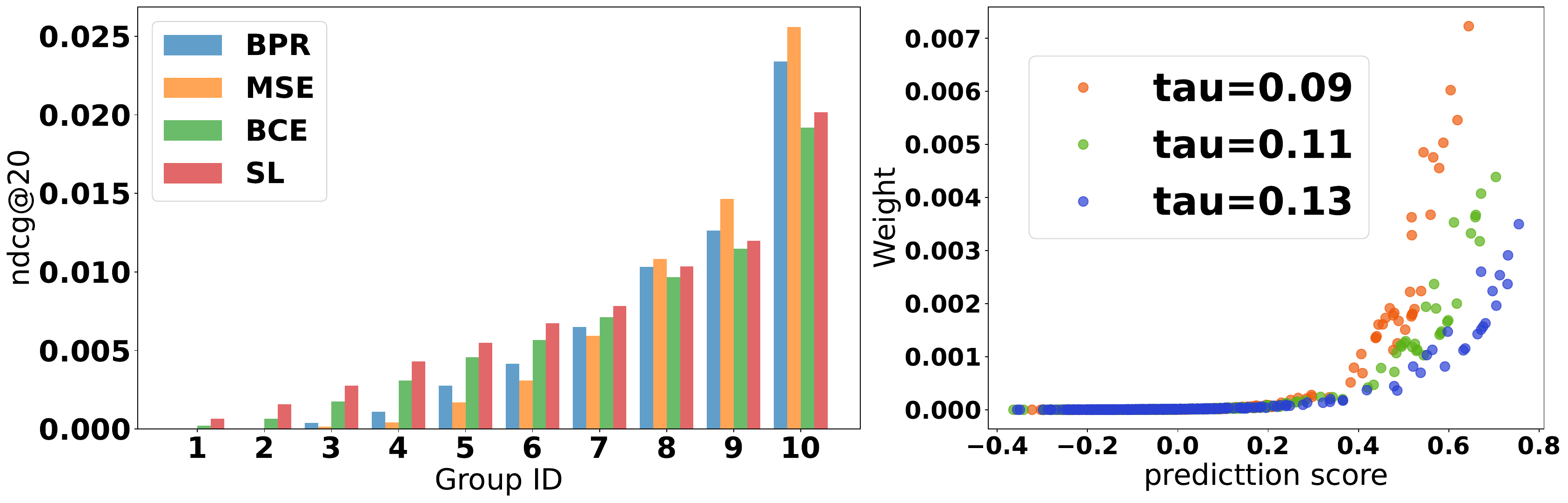}
  \caption{(a) Left: Performance comparison over different item groups among different loss functions. (b) Right: We visualize the weight of each negative sample in DRO concerning the prediction score for different $\tau$.}
  \label{fig::t_changes_performance}
\end{figure}

\subsubsection{\textbf{Exploring fairness in DRO}} 
\label{tau_fair}
To verify the fairness of SL compared to other loss functions, we follow \cite{wu2022effectiveness} and divide items into ten groups \wrt item popularity (\ie interaction frequency). The y-axis represents the cumulative score for a single group and the larger GroupID denotes the group where items are more popular. As shown in Figure \ref{fig::t_changes_performance}a, SL achieves superior performance on those unpopular items comparing with traditional loss functions (BPR, MSE, BCE). We attribute this phenomenon to the regularization of the variance of negative samples (see Equation \eqref{lam1_eq}). In other words, SL not only pursues accuracy but also fairness in recommendation.

{To establish a clear contrast, we conducted ablation study by removing the variance penalty term and, compared the performance of the standard SL loss with the revised loss. The results are presented in Figure \ref{fig::variance_w_wo_}. As can be seen, the exclusion of the variance term significantly exacerbates the unfairness of recommendation --- \ie better performance on popular groups (\eg groups 8-10) while much worse performance on unpopular groups (\eg groups 1-5).
}

\begin{figure}[t]\centering
  \includegraphics[width=0.5\linewidth]{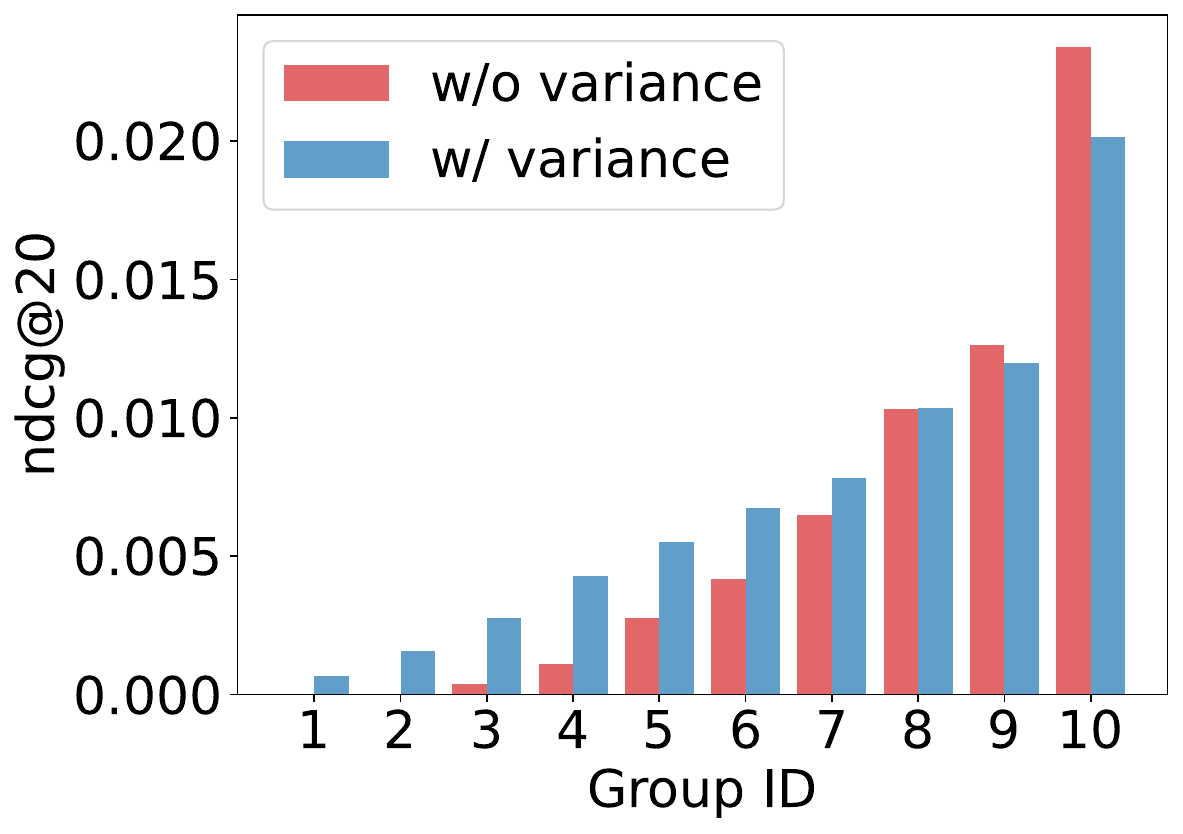}
  \caption{{Ablation study on variance terms. The item groups with larger ID suggest these items have larger popularity.}}
  \label{fig::variance_w_wo_}
\end{figure}
\section{Proposed Loss Function}
\label{sec_4}
In this section, we first analyze the necessity of positive denoising and then propose a novel loss function through the lens of DRO.
\subsection{Necessity of Positive Denoising}
In the former section, we examined the superiority of SL, particularly its robustness in negative sampling. However, it is worth noting that noise can also be present in positive samples. This is particularly relevant in recommendation systems, where clicks can be influenced by various factors, such as attractive titles or high popularity, leading to a gap between the observed positive data and the user's true preference \cite{wang2021clicks, wang2021denoising}. Training a model directly on such noisy data may result in poor performance.

To investigate this potential yet critical issue in recommendation systems, we conducted experiments where we randomly treated a certain proportion of negative items (in accordance with the interaction frequency per user) as fake positive instances and evaluated the performance of the model using SL. As illustrated in Figure \ref{fig::pos_noise_ratio}, a rise in the noise ratio corresponded to a decline in performance. We suspect that this phenomenon may be even more pronounced in real-world datasets. Therefore, we are inclined to leverage the benefits of DRO to tackle the aforementioned critical issues.

\subsection{Bilateral SoftMax Loss}
In order to enhance the robustness of the model to both noisy positive and negative data, we propose the Bilateral SoftMax Loss (BSL), which extends the benefits of SL to both the positive and negative sides. It is worth noting that SL (as seen in Eq.(\ref{sl_eq})) can be expressed as follows, comprising of two components:
  \begin{equation}
    \begin{aligned}
    \mathcal{L}_{SL}(u) &= \underbrace{\text{\fcolorbox{black}{pink}{\parbox{0.3\linewidth}{$-\mathbb{E}_{i \sim P^+_u} [f(u,i)]$}}}}_{\text{Positive Part}} + \\
    & \underbrace{{\tau}\log \mathbb E_{j\sim P^-_u}[\exp(f(u,j)/\tau)]}_{\text{Negative Part}}
    \end{aligned}
    \end{equation}
The effectiveness of SL is attributed to its negative loss structure, which involves the Log-Expectation-Exp structure (i.e., $\log\mathbb E[\exp(.)]$), equivalent to conducting distributional robust optimization over the vanilla point-wise loss (as illustrated in Lemma \ref{lam0}). Therefore, it is reasonable to incorporate the positive loss into this structure, allowing it to benefit from the robust property. Based on this intuition, the Bilateral SoftMax Loss can be formulated as follows:
\begin{equation}
  \begin{aligned}
    \mathcal{L}_{BSL}(u) &= \underbrace{\text{\fcolorbox{black}{pink}{\parbox{0.51\linewidth}{$-\tau_1\log\mathbb{E}_{i \sim P^+_u} [\exp(f(u,i)/\tau_1)]$}}}}_{\text{Positive Part}}\\
    &+ \underbrace{{\tau_2}\log \mathbb E_{j\sim P^-_u}[\exp(f(u,j)/\tau_2)]}_{\text{Negative Part}}
    \end{aligned}
  \end{equation}

It is important to note that in BSL, distinct temperatures are utilized for the positive ($\tau_1$) and negative ($\tau_2$) components. This is mainly due to the belief that in real-world datasets, the positive and negative data may exhibit different levels of noise.
\begin{figure}[t]\centering
  \includegraphics[width=0.8\linewidth]{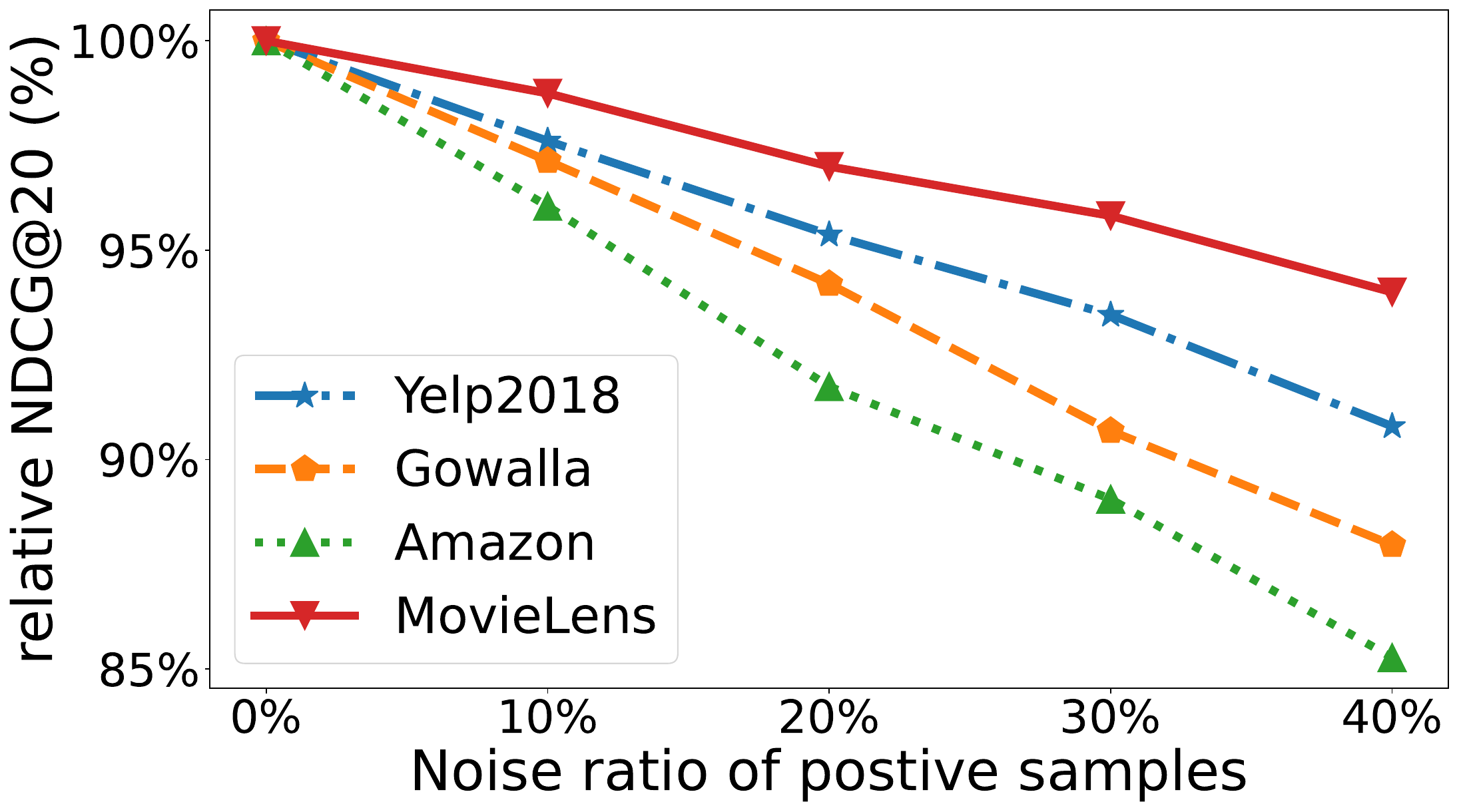}
  \caption{The performance comparison with adding a certain ratio of noisy data into the positive data.}
  \label{fig::pos_noise_ratio}
\end{figure}

      Overall, we remark that BSL satisfies the following desirable properties:
      \begin{itemize}[leftmargin=*]
        \item \textbf{Bilateral Robustness.} Thanks to the application of Log-Expectation-Exp structure on both positive and negative sides, 
        BSL possesses the capability to maintain robustness on both sides. Unlike previous research that primarily focuses on positive sample denoising \cite{wang2021clicks, wang2021denoising} or negative sample debiasing \cite{MixGCF2021, zhang2013optimizing, yang2020understanding}, our work elegantly accomplishes both objectives. In section \ref{exp_main}, we also provide empirical validation of its robustness and superior performance.
        \item \textbf{Model-Agnostic.} BSL is model-agnostic and straightforward to implement. Notably, BSL only requires one additional line of code compared to the standard SL (as seen in Algorithm \ref{alg:code2}). In our experiments, we also apply BSL to common recommendation models such as MF \cite{MF}, NGCF \cite{NGCF}, LightGCN \cite{he2020lightgcn}, SGL  \cite{wu2021self}, SimSGL \cite{yu2022graph}  and LightGCL \cite{caisimple} to validate its effectiveness.
        
      \end{itemize}
      \textbf{}

\begin{table}[!t]
    \centering
    \caption{Statistics of the datasets.}
    \resizebox{1.0\linewidth}{!}{
      \begin{tabular}{c|r|r|r|c}
        \toprule
        Dataset     & \#Users & \#Items & \#Interactions & Density \\ \midrule
        Amazon & 192,403 & 63,001 & 1,689,188    & 0.014\% \\
        Yelp2018    & 31,668 & 38,048 & 1,561,406    & 0.130\% \\ 
        Gowalla     & 29,858 & 40,981 & 1,027,370    & 0.084\% \\ 
        Movielens-1M & 6,022 & 3,043 & 995,154 & 5.431\% \\
        \bottomrule
        \end{tabular}}
    \label{data_statistics}
\end{table}

\section{Experiments}
\label{exp_main}
In this section, we present various experiments to verify the superiority of our model. We aim to answer the following research questions:
\begin{itemize}[leftmargin=*]
  \item \textbf{RQ1:} How do BSL and SL perform compared to other loss functions?
  \item \textbf{RQ2:} Can BSL and SL maintain stability when we increase the number of noisy negative instances?
  \item \textbf{RQ3:} Is BSL more resilient to positive noise than SL?
  \item {\textbf{RQ4:} How does BSL perform when hyper-parameters are changed ?}
\end{itemize}

\subsection{Experimental Setup}
\textbf{Datasets and Evaluation Protocol.} We utilize four publicly available datasets, namely Yelp2018\cite{he2020lightgcn}, Amazon \cite{huang2021mixgcf}, Movielens \cite{yu2020graph}, and Gowalla \cite{he2017neural_a}, to evaluate the effectiveness of our model. To ensure a fair comparison, we adhere to the traditional settings on collaborative filtering and employ the same data split as they did \cite{he2020lightgcn}. Table \ref{data_statistics} showcases the statistics of the aforementioned datasets. Additionally, we use conventional metrics, Recall@20 and NDCG@20, to assess the recommendation performance.

\begin{table*}[h]
  \centering
  \caption{Overall performance comparison. (+X) denotes the loss functions. The best result is bolded and the runner-up is underlined.} 
  \resizebox{1.02\textwidth}{!}{
    \begin{tabular}{c |l|cc|cc|cc|cc}
      \toprule
      \multicolumn{2}{c|}{\multirow{2}{*}{Model}} & \multicolumn{2}{c|}{Amazon} & \multicolumn{2}{c|}{Yelp2018}     & \multicolumn{2}{c|}{Gowalla}      & \multicolumn{2}{c}{Movielens-1M}  \\ \cline{3-10} 
      \multicolumn{2}{c|}{}      & Recall@20       & NDCG@20         & Recall@20       & NDCG@20         & Recall@20           & NDCG@20         & Recall@20       & NDCG@20         \\ \midrule
      
      \multicolumn{1}{c|}{{WWW'2017}} &CML                 &     0.0362      &    0.0165       & 0.0622          & \underline{0.0536}          & 0.1670          & 0.1292          & 0.1730          & 0.1563          \\
      \multicolumn{1}{c|}{SIGIR'2019}&ENMF                   &     0.0405      &      0.0182      & 0.0624          & 0.0515          & 0.1523          & 0.1315          & 0.2315          & 0.2069          \\ \hline
      \multicolumn{1}{c|}{SIGIR'2020}&NIA-GCN                &      -    &       -    & 0.0599          & 0.0491          & 0.1359          & 0.1106          & 0.2359          & 0.2242          \\
      \multicolumn{1}{c|}{AAAI'2020}&LR-GCCF                &     0.0326      &    0.0158        & 0.0561          & 0.0343          & 0.1519          & 0.1285          & 0.2231          & 0.2124          \\
      \multicolumn{1}{c|}{SIGIR'2020}&DGCF                   &      0.0365     &    0.0164   & 0.0654         & 0.0534        & 0.1842         & 0.1561          & 0.2640          & 0.2504          \\ \hline
      \multicolumn{1}{c|}{SIGIR'2021}&SGL       & 0.0454 & 0.0209 &   0.0690 & 0.0570 &   0.1798 & 0.1524 &  0.2528 & 0.2401 \\ 
      \multicolumn{1}{c|}{CIKM'2021}&SimpleX       & - & - &   0.0701 & 0.0575 &   0.1872 & 0.1557 &  0.2802 & {0.2670} \\ 
      \multicolumn{1}{c|}{CIKM'2021}&UltraGCN               & 0.0421 & 0.0195 & 0.0683 & 0.0561 &0.1862 & 0.1580 & 0.2787 & 0.2642 \\
      \multicolumn{1}{c|}{SIGIR'2022}&SimSGL    &   0.0486 & 0.0225 &   0.0721 & 0.0601 &   0.1834 & 0.1550 &  0.2689 & 0.2520 \\ 
    
      \multicolumn{1}{c|}{WWW'2022}&NCL  &  0.0426 & 0.0195 &   0.0664 & 0.0552 &   0.1760 & 0.1506 &  0.2623 & 0.2496 \\ 
      \hline
      \multirow{5}{*}{\shortstack{{MF}\\UAI'2009}}& +  BPR                 &       0.0268    &   0.0115        & 0.0549          & 0.0445          & 0.1616          & 0.1366          & 0.2153          & 0.2175          \\
      \multicolumn{1}{c|}{} &+ BCE                 &       0.0276 & 0.0120 & 0.0513 & 0.0415 & 0.1492 & 0.1251 & 0.2329 & 0.2244 \\
      \multicolumn{1}{c|}{}&+ MSE                 &       0.0321 & 0.0135 & 0.0415 & 0.0337 & 0.1160 & 0.1005 & 0.2216 & 0.2023 \\
      \multicolumn{1}{c|}{}&+ SL             &  \cellcolor{lightblue}0.0446  &    \cellcolor{lightblue}0.0195 &  \cellcolor{lightblue}0.0718 &  \cellcolor{lightblue}0.0585 &  \cellcolor{lightblue}0.1760 &  \cellcolor{lightblue}0.1399  &  \cellcolor{lightblue}0.2786 & \cellcolor{lightblue} 0.2633        \\ 
      \multicolumn{1}{c|}{}& + BSL             &  \cellcolor{lightblue}0.0476 & \cellcolor{lightblue} 0.0217  &  \cellcolor{lightblue}0.0725&  \cellcolor{lightblue}{0.0599} &  \cellcolor{lightblue}\underline{0.1882} &  \cellcolor{lightblue}\underline{0.1584} &  \cellcolor{lightblue}\underline{0.2804}&  \cellcolor{lightblue}\textbf{0.2681}   \\ 
    
      \hline
      \multirow{5}{*}{\shortstack{NGCF\\SIGIR'2019}}& + BPR               &      0.0235 & 0.0076 & 0.0579 & 0.0477 & 0.1570 & 0.1327 & 0.2513 & 0.2511 \\ 
      & + BCE                 &      0.0266 & 0.0107 & 0.0498 & 0.0401 & 0.1410 & 0.1169 & 0.2421 & 0.2292 \\ 
      & + MSE                 &      0.0234 & 0.0097 & 0.0391 & 0.0319 & 0.0837 & 0.0720 & 0.2376 & 0.2135 \\
      &+ SL            & \cellcolor{lightblue}0.0468 &  \cellcolor{lightblue}0.0218 & \cellcolor{lightblue}0.0711  & \cellcolor{lightblue}0.0589 & \cellcolor{lightblue}0.1674 & \cellcolor{lightblue}0.1399 & \cellcolor{lightblue}0.2789 & \cellcolor{lightblue}0.2632       \\ 
    &+ BSL              &  \cellcolor{lightblue}0.0491   &  \cellcolor{lightblue}0.0232 &  \cellcolor{lightblue}{0.0732} &  \cellcolor{lightblue}{0.0610}&  \cellcolor{lightblue}0.1829 &  \cellcolor{lightblue}0.1513 &  \cellcolor{lightblue}\textbf{0.2808}&  \cellcolor{lightblue}\underline{0.2676}   \\ 
    
      \hline
      \multirow{5}{*}{\shortstack{LGN\\SIGIR'2020}}&+BPR               &   0.0399 & 0.0180 & 0.0649 & 0.0530 & 0.1830 & 0.1554 & 0.2576 & 0.2427 \\
      &+BCE                 &       0.0386 & 0.0177 & 0.0557 & 0.0457 & 0.1470 & 0.1271 & 0.2484 & 0.2412 \\
      &+MSE                 &       0.0336 & 0.0150 & 0.0463 & 0.0366 & 0.1347 & 0.1080 & 0.2488 & 0.2291 \\
      &+SL         &      \cellcolor{lightblue}\underline{0.0519}      &     \cellcolor{lightblue}\underline{0.0248}     &  \cellcolor{lightblue}\underline{0.0736}&  \cellcolor{lightblue}\underline{0.0612}&  \cellcolor{lightblue}0.1878         &  \cellcolor{lightblue}0.1577         &  \cellcolor{lightblue}0.2792     & \cellcolor{lightblue}0.2627  \\
      &+BSL              &    \cellcolor{lightblue}\textbf{0.0521}     &      \cellcolor{lightblue}\textbf{0.0249}     &  \cellcolor{lightblue}\textbf{0.0744}    &  \cellcolor{lightblue}\textbf{0.0617}         &  \cellcolor{lightblue}\textbf{0.1893}   &  \cellcolor{lightblue}\textbf{0.1591}    &  \cellcolor{lightblue}0.2799    &  \cellcolor{lightblue}0.2627        \\ 
    
     
    

      \bottomrule
      \end{tabular}
  }
  \label{overall_res}
  \end{table*}

\begin{table*}[h]
  \centering
  \caption{{Overall performance comparison. (+X) denotes the loss functions. The best result is bolded and the runner-up is underlined.}} 
  \resizebox{1.0\textwidth}{!}{
    \begin{tabular}{c |l|ll|ll|ll|ll|c}
      \toprule
      \multicolumn{2}{c|}{\multirow{2}{*}{Model}} & \multicolumn{2}{c|}{Amazon} & \multicolumn{2}{c|}{Yelp2018}     & \multicolumn{2}{c|}{Gowalla}      & \multicolumn{2}{c|}{Movielens-1M}  & \multirow{2}{*}{Avg.} \\ \cline{3-10} 
      \multicolumn{2}{c|}{}      & Recall@20       & NDCG@20         & Recall@20       & NDCG@20         & Recall@20           & NDCG@20         & Recall@20       & NDCG@20    &  \\ \midrule
      
      \multirow{3}{*}{\shortstack{SGL\\ (\citeauthor{wu2021self}, \citeyear{wu2021self})}}&  -      & 0.0454 & 0.0209 &   0.0690 & 0.0570 &   0.1798 & 0.1524 &  0.2528 & 0.2401 & \\ 
      & +SL       & $ 0.0513^{\color{+}+ 13.14\%}$ & $ 0.0241^{\color{+}+ 15.45\%}$ & $ \underline{0.0736}^{\color{+}+ 6.68\%}$ & $\underline{ 0.0612}^{\color{+}+ 7.35\%}$ & $ 0.1872^{\color{+}+ 4.12\%}$ & $ 0.1568^{\color{+}+ 2.89\%}$ & $ 0.2766^{\color{+}+ 9.41\%}$ & $ 0.2579^{\color{+}+ 7.41\%}$ & \color{+}+ 8.31\%  \\
      & +BSL      & $ 0.0513^{\color{+}+ 13.16\%}$ & $ 0.0242^{\color{+}+ 15.60\%}$ & $ \textbf{0.0741}^{\color{+}+ 7.35\%}$ & $ \textbf{0.0615}^{\color{+}+ 7.86\%}$ & $ \textbf{0.1890}^{\color{+}+ 5.12\%}$ & $ \textbf{0.1592}^{\color{+}+ 4.46\%}$ & $ \underline{0.2787}^{\color{+}+ 10.25\%}$ & $ \underline{0.2621}^{\color{+}+ 9.16\%}$ & \color{+}+ 9.12\%  \\
      \hline
      \multirow{3}{*}{\shortstack{SimSGL\\ (\citeauthor{yu2022graph}, \citeyear{yu2022graph})}}& -    &   0.0486 & 0.0225 &   0.0721 & 0.0601 &   0.1834 & 0.1550 &  0.2689 & 0.2520 & \\ 
      & +SL       & $ 0.0515^{\color{+}+ 5.84\%}$ & $ 0.0239^{\color{+}+ 6.35\%}$ & $ 0.0733^{\color{+}+ 0.89\%}$ & $ 0.0608^{\color{+}+ 1.64\%}$ & $ 0.1835^{\color{+}+ 0.05\%}$ & $ 0.1552^{\color{+}+ 0.13\%}$ & $ 0.2752^{\color{+}+ 2.34\%}$ & $ 0.2565^{\color{+}+ 1.79\%}$ & \color{+}+ 2.38\%  \\ 
      & +BSL      & $ 0.0514^{\color{+}+ 5.78\%}$ & $ 0.0239^{\color{+}+ 6.35\%}$ & $ 0.0736^{\color{+}+ 1.32\%}$ & $ 0.0611^{\color{+}+ 2.17\%}$ & $ 0.1834^{\color{+}+ 0.00\%}$ & $ 0.1558^{\color{+}+ 0.52\%}$ & $ \textbf{0.2796}^{\color{+}+ 3.98\%}$ & $ \textbf{0.2676}^{\color{+}+ 6.19\%}$ & \color{+}+ 3.29\%  \\
      \hline
      \multirow{3}{*}{\shortstack{LightGCL\\ (\citeauthor{caisimple}, \citeyear{caisimple})}}& -    &  0.0489 & 0.0224 &   0.0697 & 0.0573 &   0.1867 & 0.1562 &  0.2316 & 0.2287 & \\
      & +SL       & $ \underline{0.0530}^{\color{+}+ 8.49\%}$ & $ \underline{0.0246}^{\color{+}+ 9.86\%}$ & $ 0.0698^{\color{+}+ 0.11\%}$ & $ 0.0581^{\color{+}+ 1.38\%}$ & $ 0.1805^{\color{-} -1.58\%}$ & $ 0.1539^{\color{+}+ 0.26\%}$ & $ 0.2742^{\color{+}+ 18.39\%}$ & $ 0.2562^{\color{+}+ 12.02\%}$ & \color{+}+ 6.12\%  \\
      & +BSL      & $ \textbf{0.0533}^{\color{+}+ 8.94\%}$ & $ \textbf{0.0247}^{\color{+}+ 10.22\%}$ & $ 0.0714^{\color{+}+ 2.39\%}$ & $ 0.0594^{\color{+}+ 3.70\%}$ & $ 0.1837^{\color{+}+ 0.16\%}$ & $ 0.1553^{\color{+}+ 1.17\%}$ & $ 0.2767^{\color{+}+ 19.47\%}$ & $ 0.2597^{\color{+}+ 13.55\%}$ & \color{+}+ 7.45\%  \\
      
      \bottomrule
      \end{tabular}
  }
  \label{overall_res_appendix}
\end{table*}

\begin{figure*}[t]\centering
  \includegraphics[width=1.0\linewidth]{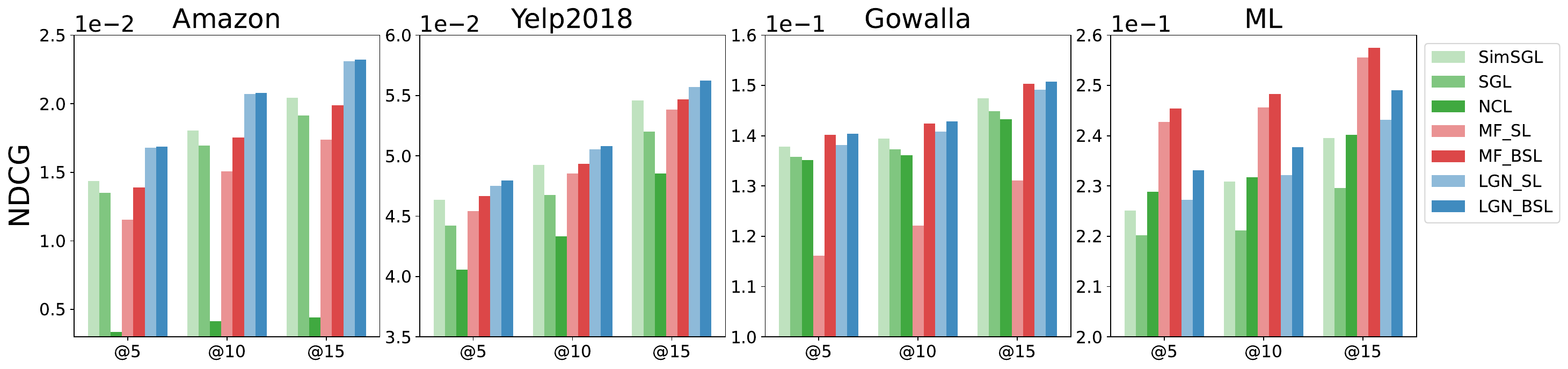}
  \caption{{Performance comparison \wrt different metric setting.}}
  \label{fig::ndcg_5_10_15}
\end{figure*}

\textbf{Baselines.} As the suggested loss function is model-agnostic, we select three fundamental and efficient models (MF\cite{MF}, NGCF\cite{NGCF}, and LightGCN\cite{he2020lightgcn}) as backbones and substitute their loss functions with our proposed losses. Additionally, we implement other loss functions for comparison purposes. Moreover, to further verify the efficacy of our proposal, we compare it with various models from three different categories:

(1) \textit{MF-based method}: ENMF\cite{ENMF} leverages a simple Neural Matrix Factorization architecture without sampling.

(2) \textit{metric learning-based method}: CML\cite{CML}, which combines metric learning and collaborative filtering, has the advantage of learning user-user and item-item similarity.

(3) \textit{GCN-based methods}:
\begin{itemize}
  \item NGCF\cite{NGCF} was the pioneering work that integrated the bipartite graph structure (GCN) into the embedding process.
  \item NIA-GCN\cite{NIA-GCN} explicitly considered the relational information in the neighborhood and employed a novel neighbor-aware graph attention layer.
  \item LR-GCCF\cite{LR-GCCF} revisited GCN based CF models from non-linear activations and proposed a residual network structure to alleviate the oversmoothing problem.
  \item LightGCN\cite{he2020lightgcn} introduced a novel GCN structure that is more concise and suitable for recommendation systems. It is also one of the most commonly used baselines in the field of recommender systems.
  \item SGL\cite{wu2021self} explored self-supervised learning on the user-item graph and proposed a new learning paradigm based on LightGCN.
  \item Ultra-GCN\cite{mao2021ultragcn} proposed an ultra-simplified formulation of GCNs, which skips infinite layers of message passing for efficient recommendation.
  \item SimpleX\cite{mao2021simplex} demonstrated the significance of the choice of loss function and incorporated the cosine contrastive loss into a simple unified CF model.
  \item SimSGL\cite{yu2022graph} revealed that graph augmentations play a trivial role and instead added uniform noises to the embedding space to create contrastive views.
  \item NCL\cite{lin2022improving} focused on the influence of data sparsity in real scenarios and proposed a novel contrastive learning approach that incorporates potential neighbors into contrastive pairs.
  \item DGCF\cite{DGCF} devised a new model to disentangle the finer granularity of user intents, offering the advantages of disentangling user intents and interpretability of representations.
\end{itemize}

  \textbf{Parameter Settings.} To ensure a fair comparison, we set the embedding size to 64 for all compared models, and the initialization is unified using Xavier \cite{glorot2010understanding}. We perform a grid search to determine the optimal parameter settings for each model. Specifically, we tune the learning rate among $\{1e^{-3}, 5e^{-3}, 1e^{-4}\}$, and search for the coefficient of the $L_2$ regularization term in $\{1e^{-9},1e^{-8},...,1e^{-1}\}$. Regarding the backbone of NGCF and LightGCN, we tune the number of layers among \{1,2,3\}, with or without using dropout to prevent over-fitting. With regards to SL and BSL, we search for temperatures with an interval of $0.10$ in the range [0.05, 1.0]. We also vary the number of negative samplings among {200, 400, 800, 1500}. For the methods under comparison, we meticulously adhere to the hyperparameter settings delineated in the original paper, while conducting a comprehensive grid search to identify the optimal configuration in our experiments. Please refer to the Appendix for more details.

  \subsection{Performance Comparison (RQ1)}
  In this subsection, we begin to analyze the superiority of our loss, as compared with other baselines.
  \begin{itemize}[leftmargin=*]
    \item As demonstrated in Table \ref{overall_res}, SL and BSL consistently outperform other loss functions by a significant margin on all four datasets and three backbone models. Moreover, even when applied to basic models such as MF and LightGCN, SL and BSL can outperform state-of-the-art methods. We attribute this to the efficacy of SL and BSL in mitigating negative distribution shifts. 
    \item Our proposed BSL consistently outperforms SL in all conditions, empirically verifying the necessity of denoising over positive interactions. By upgrading the positive component of SL to an advanced Log-Expectation-Exp structure, BSL exhibits better robustness against positive noise.
    \item Notably, due to the variations among datasets and baseline models, the improvements may exhibit slight disparities. For instance, in the case of Gowalla, while MF-SL performs comparably with MF-BPR, MF-BSL improves the performance by 16.0\% and 16.5\%, in terms of Recall@20 and NDCG@20 respectively. We suspect this phenomenon is due to the presence of more positive noise in the Gowalla dataset, which hinders the effectiveness of SL.
    \item { We incorporated additional experiments on more SOTA models including  SGL (SIGIR'21) \cite{wu2021self}, SimSGL (SIGIR'22) \cite{yu2022graph}  and LightGCL (ICLR'23) \cite{caisimple}). The results are presented in Table \ref{overall_res_appendix}. Here We applied both SL and BSL losses to those SOTA methods and observed that both losses enhance model performance, with BSL yielding more significant improvements. These results are in alignment with our analyses and demonstrate the effectiveness of the proposed loss. }
    \item {As shown in Figure \ref{fig::ndcg_5_10_15}, our analysis confirms that, across the metrics of NDCG@5, NDCG@10, and NDCG@15, the integration of BSL consistently enhances model performance. This enhancement is substantial enough to allow basic models (\eg MF, lightGCN) to surpass SOTA models (\eg SimSGL, NCL). }
\end{itemize}

\subsection{Performance with Noisy Negative Data (RQ2)}
In this section, we verify the resilience of SL and BSL against noise on negative data. We conduct two sets of experiments: 1) Fixing the number of negative instances and allowing the negative sampler to draw a certain ratio of positive data. 2) Increasing the number of uniformly sampled instances using the conventional method. Strategy 1 may lead to more false negatives, while strategy 2 increases the quantity of noisy negative instances. We perform these experiments on MF and observe similar results on other backbone models.

\subsubsection{Impact of noisy ratio in negative samples} 
\begin{itemize}[leftmargin=*] 
  \item Figure \ref{fig::pos_ratio} shows that noisy data negatively affects loss function performance. However, SL and BSL can handle noise in negative sampling by adjusting the parameter $\tau$, leading to better performance compared to other loss functions. Moreover, as the proportion of noisy negatives increases, the optimal $\tau$ also increases, aligning with the analysis in Section \ref{Empirical}.
  \item Furthermore, an interesting finding is that increasing the noise ratio on Yelp2018 can unexpectedly boost the performance of MSE and BCE. This anomaly arises from the inherent instability and varying performance of point-wise loss functions across datasets and scenarios. However, SL and BSL consistently outperform other loss functions and demonstrate greater stability.

\begin{figure}[t]\centering
    \includegraphics[width=1.0\linewidth]{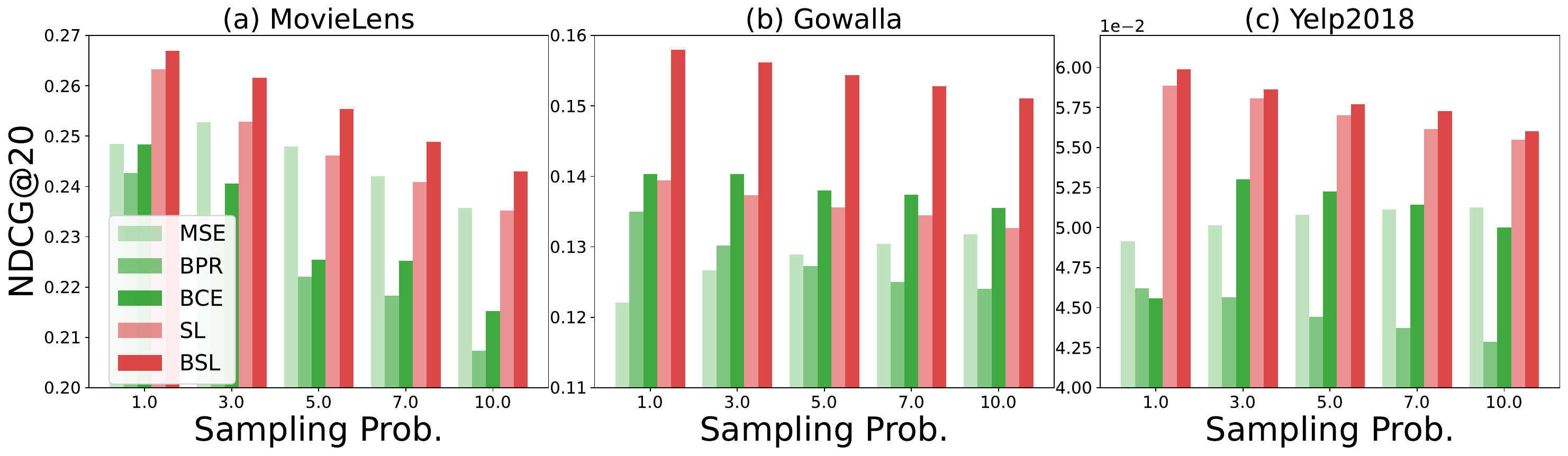}
    \caption{Performance comparison \wrt the ratio of false negative samples. A grid search is conducted to confirm the optimal parameter setting for each model.}
    \label{fig::pos_ratio}
\end{figure}

\subsubsection{Impact of sampling number}
\begin{itemize}[leftmargin=*]
  \item Firstly, as depicted in Figure \ref{fig::num_neg}, traditional loss functions such as BCE, BPR, and MSE can be negatively impacted by the number of negative samples. Specifically, on the MovieLens dataset, MSE continues to decline while BCE experiences fluctuations with the addition of negative samples. We attribute this to the fact that MovieLens is a relatively small and dense dataset, with only 3k items. Excess sampling numbers (greater than 512) inevitably increase the number of false negatives. However, in terms of distributionally robust optimization, SL and BSL exhibit extremely stable performance.
  \item Secondly, we have observed that BSL consistently outperforms SL, which validates the importance of leveraging the "worst-case" optimization over positive samples. When examining the Gowalla dataset, although SL displays a rising and gradually stabilizing trend, it is unable to achieve superior performance against other baseline models. Instead, BSL exhibits the same trend as SL but effectively addresses the issue of false positive samples on Gowalla, enabling it to demonstrate superiority. As a result, we argue that our proposed simple yet effective BSL is crucial for recommendation systems.
\end{itemize}

\begin{figure}[t]\centering
  \includegraphics[width=1.0\linewidth]{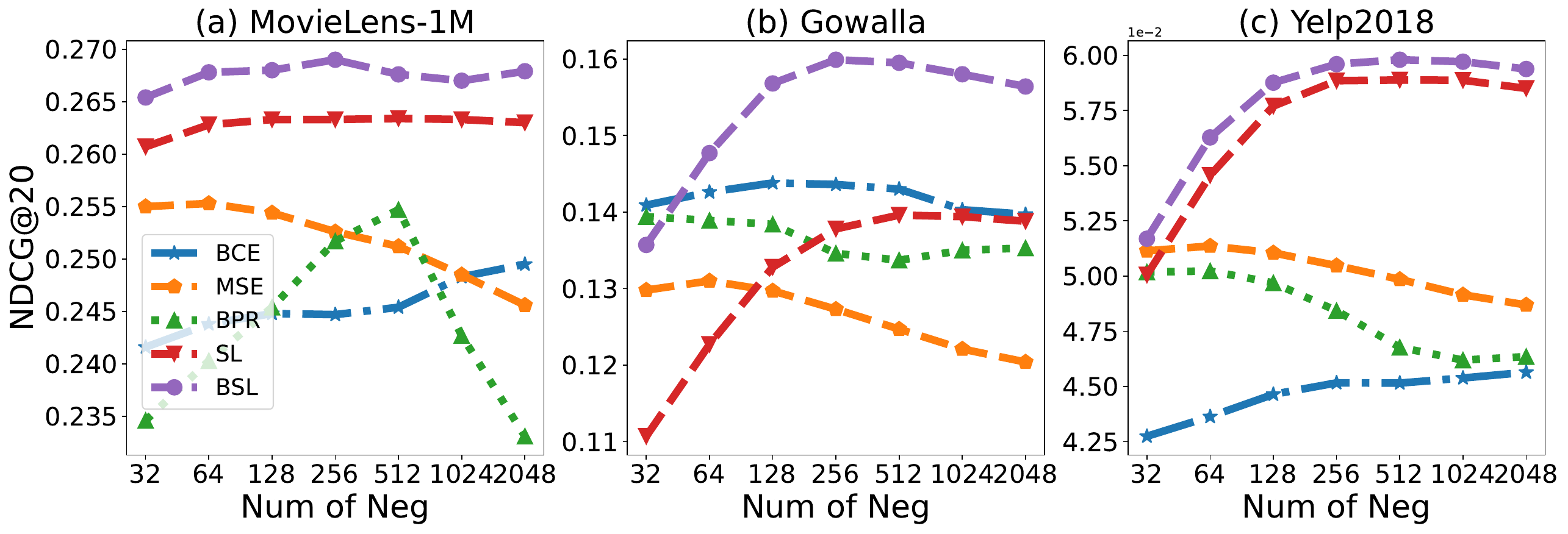}
  \caption{Performance comparison \wrt the number of negative samples. A grid search is conducted to confirm the optimal parameter setting for each model.}
  \label{fig::num_neg}
  \vspace{-.8em}
\end{figure}

\end{itemize}

\begin{table*}[t]
  \centering
  \caption{Overall performance comparison under different noisy ratio in positive data.}
  \resizebox{1.0\textwidth}{!}{
    \begin{tabular}{cl|cc|cc|cc|cc}
      \toprule
      \multirow{2}{*}{ratio}&\multirow{2}{*}{Model} & \multicolumn{2}{c|}{Amazon} & \multicolumn{2}{c|}{Yelp2018}     & \multicolumn{2}{c|}{Gowalla}      & \multicolumn{2}{c}{Movielens-1M}  \\ \cline{3-10} 
                            & & Recall@20       & NDCG@20         & Recall@20       & NDCG@20         & Recall@20           & NDCG@20         & Recall@20       & NDCG@20         \\ \midrule
      \multirow{3}{*}{10\%} & MF-SL &  0.0465 & 0.0203 & 0.0704 & 0.0576 &  0.1719 & 0.1356 &  0.2755 & 0.2600  \\
      & {MF-BSL} &	\textbf{0.0482} &\textbf{0.0217} & \textbf{0.0711} & \textbf{0.0589} & \textbf{0.1844} & \textbf{0.1566} &\textbf{0.2796} &  \textbf{0.2659} \\
      & {\%Improv.} &	\color{+}+3.66\% & \color{+}+6.90\% & \color{+}+1.00\% & \color{+}+2.26\% & \color{+}+7.27\% & \color{+}+15.49\% & \color{+}+1.49\% &  \color{+}+2.27\% \\
      \hline
      \multirow{3}{*}{20\%} & MF-SL & 0.0442 & 0.0191	& 0.0686 & 0.0563 &  0.1667 & 0.1315 & 0.2716 & 0.2554  \\
      & MF-BSL & \textbf{0.0465} & \textbf{0.0209}	& \textbf{0.0700 }& \textbf{0.0577} &  \textbf{0.1815} & \textbf{0.1542} & \textbf{0.2772 }& \textbf{0.2635 } \\
      & {\%Improv.} &	\color{+}+5.20\% & \color{+}+9.42\% & \color{+}+2.04\% & \color{+}+2.49\% & \color{+}+8.88\% & \color{+}+17.26\% & \color{+}+2.06\% & \color{+}+3.17\% \\
      \hline
      \multirow{3}{*}{30\%} & MF-SL & 0.0421 & 0.0182	& 0.0675 & 0.0551 & 0.1611 & 0.1266 & 0.2681 & 0.2523  \\
      & MF-BSL &\textbf{0.0446} &  \textbf{0.0199}	 & \textbf{0.0687} &\textbf{0.0570} & \textbf{0.1786} & \textbf{0.1517 }& \textbf{0.2742} & \textbf{0.2622}  \\
      & {\%Improv.} &	\color{+}+5.94\% & \color{+}+9.34\% & \color{+}+1.78\% & \color{+}+3.45\% & \color{+}+10.86\% & \color{+}+19.83\% & \color{+}+2.28\% &  \color{+}+3.51\% \\
      \hline
      \multirow{3}{*}{40\%} & MF-SL & 0.0403 & 0.0172 & 0.0657 & 0.0536 & 0.1571 & 0.1228 & 0.2645 & 0.2475  \\
      & MF-BSL & \textbf{0.0428} & \textbf{0.0192}	& \textbf{0.0681} & \textbf{0.0560} & \textbf{0.1758 } & \textbf{0.1495} &  \textbf{0.2717} & \textbf{0.2578}  \\
      & {\%Improv.} &	\color{+}+6.20\% & \color{+}+11.63\% & \color{+}+3.65\% & \color{+}+4.48\% & \color{+}+11.90\% & \color{+}+21.74\% & \color{+}+2.72\% &  \color{+}+4.16\% \\
      \bottomrule
      \end{tabular}
  }
  \label{tab:pos_noise}
\end{table*}

  \subsection{Performance with Noisy Positive data (RQ3)}
  In this section, we aim to validate the resilience of BSL to positive noise. To this end, we contaminate the positive instances by introducing a certain proportion of randomly sampled negative items, ranging from 10\% to 40\%, while keeping the test set unchanged. Additionally, we utilize t-SNE to visualize the item embeddings and analyze the effect of false positive data. The MF backbone is employed, and our observations are as follows:

  \subsubsection{Performance Comparison}
  \begin{itemize}[leftmargin=*]
    \item Undoubtedly, the introduction of noise into the positive data can have a detrimental impact on model performance. However, we contend that BSL possesses a significant ability to withstand such noise. As demonstrated in Table \ref{tab:pos_noise}, even when the ratio of noise training data reaches 40\%, the decrease in performance is not excessively large, ranging from 5\% to 15\%. We attribute this to two factors: 1) The Log-Expectation-Exp structure in the positive loss enables the model to optimize under various distributions with perturbations, resulting in better robustness to positive distribution shifts. 2) The separate temperatures allow the model to adapt to different levels of positive and negative noise.
    \item The performance degradation of our proposed BSL is consistently lower than that of SL, which validates the effectiveness of denoising. For instance, when contaminated with 40\% noisy interactions on Gowalla, SL experiences a significant drop in performance (from 0.1399 to 0.1228). In contrast, the robustness of BSL is enhanced by increasing the ratio $\frac{\tau_1}{\tau_2}$ and achieving a score of 0.1496 (from 0.1584 to 0.1496). Additionally, BSL demonstrates greater improvement over SL in scenarios with more noise. Specifically, with regards to Gowalla, the performance of BSL with 40\% additional noisy interactions still surpasses that of SL with a noise-free dataset. This outcome highlights the necessity of denoising positive samples.
  \end{itemize}

  \begin{figure}[t]\centering
    \includegraphics[width=1.0\linewidth]{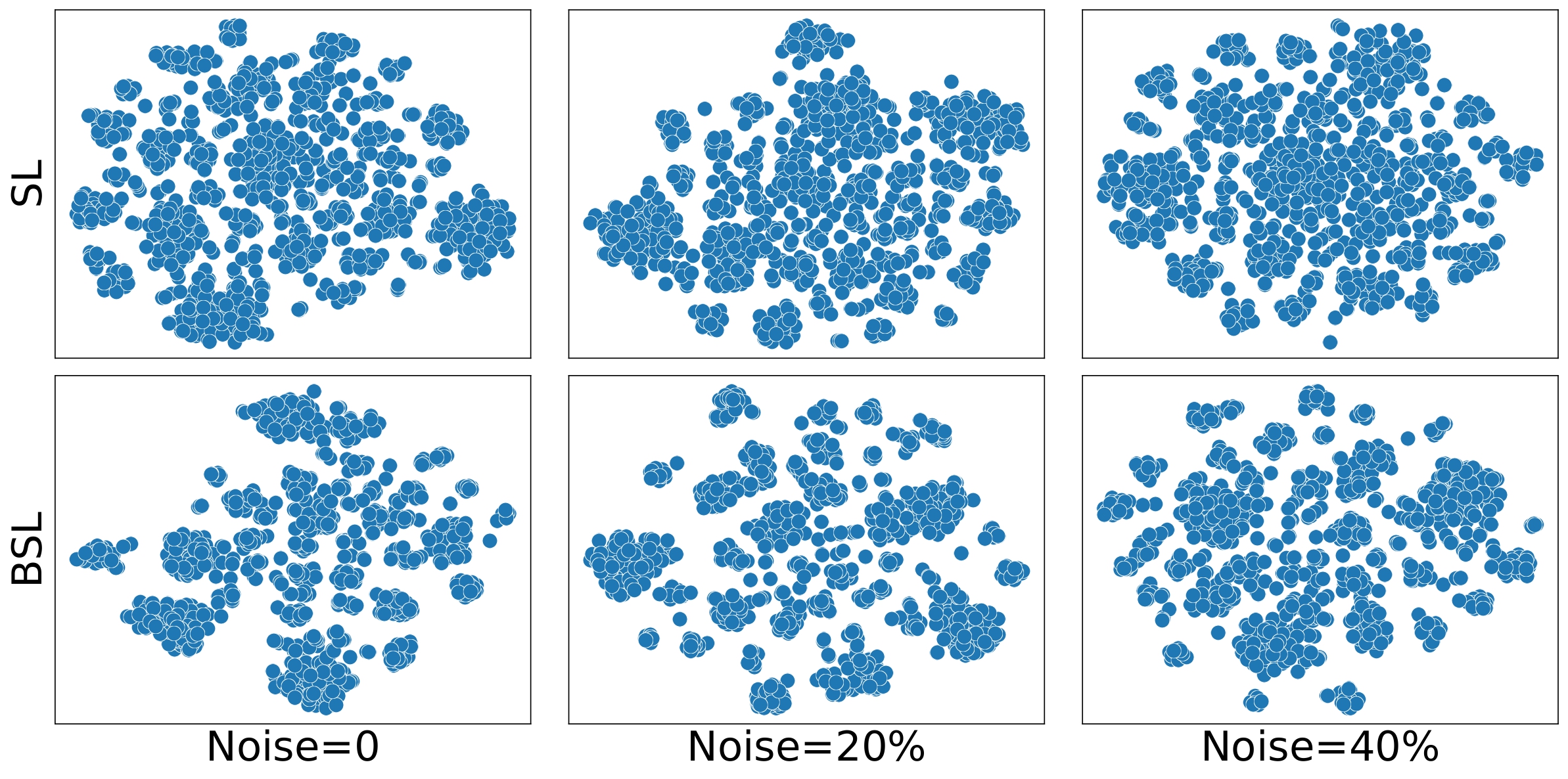}
    \caption{{t-SNE Visualization on Gowalla with additional noise data. BSL leads to better group-wise separation than SL in each noise case.}}
    \label{fig::tsne}
  \end{figure}

  \begin{figure}[t]\centering
    \includegraphics[width=1.0\linewidth]{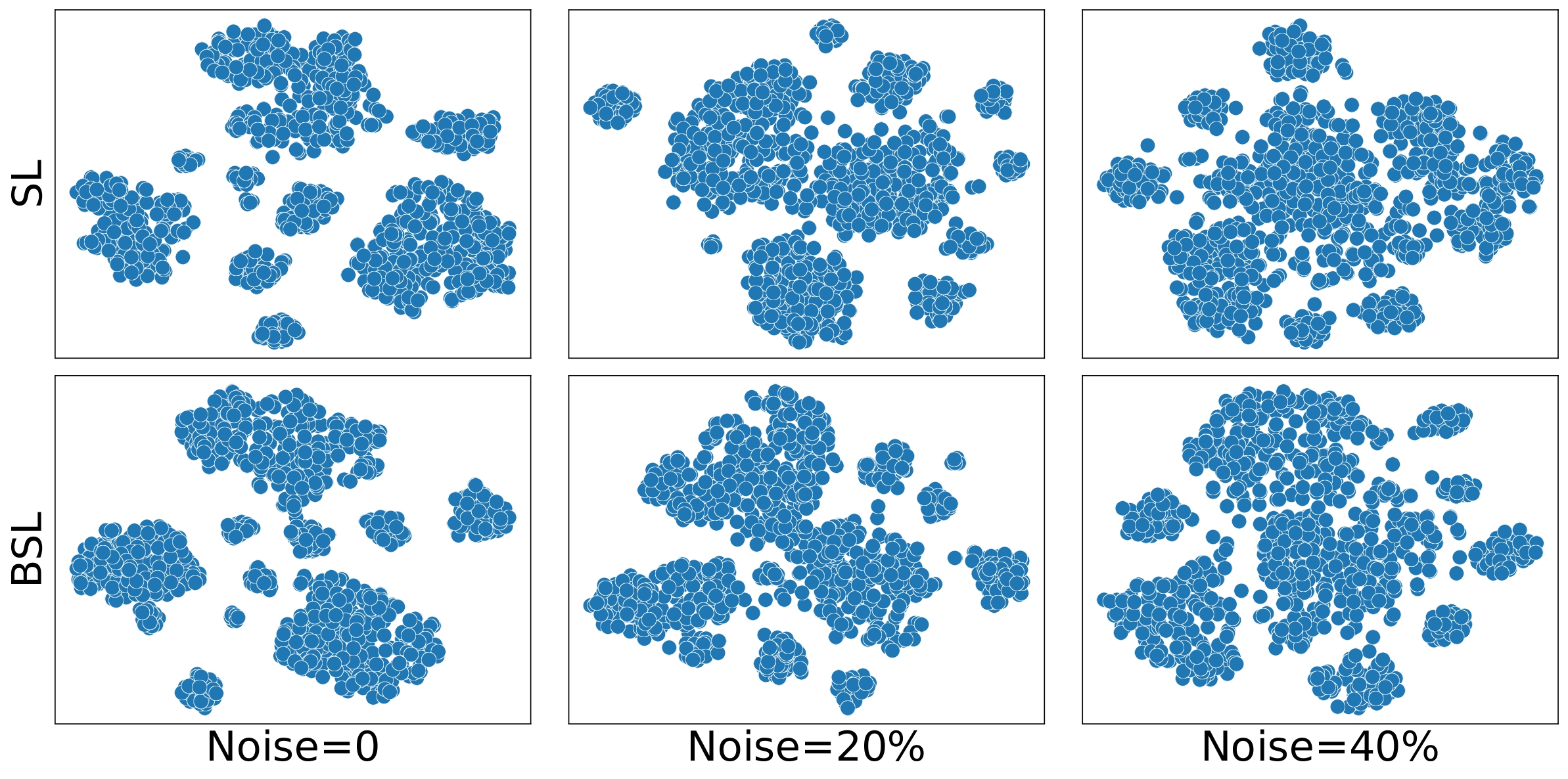}
    \caption{{t-SNE Visualization on Yelp2018 with additional noise data. BSL leads to better group-wise separation than SL in each noise case.}}
    \label{fig::yelp_tsne}
  \end{figure}

  \begin{figure*}[h]\centering
    \includegraphics[width=1.0\linewidth]{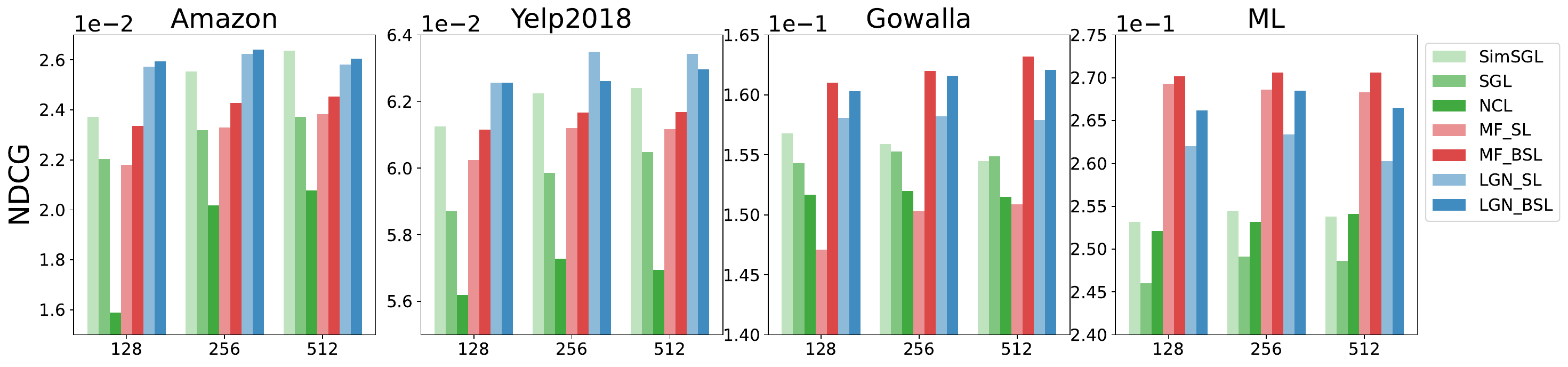}
    \caption{{ Performance comparison \wrt the different embedding dimension.} }
    \label{fig::ndcg_DIM}
    \vspace{-1.5em}
  \end{figure*}
  \begin{figure*}[t]\centering
    \includegraphics[width=1.0\linewidth]{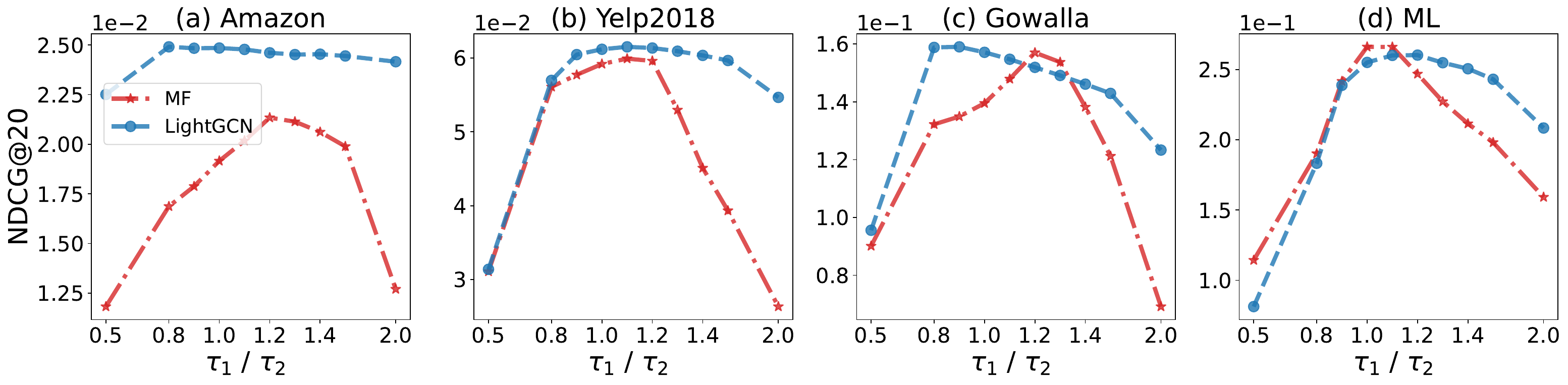}
    \caption{{Performance comparison \wrt the ratio of $\tau_1 / \tau_2$.}}
    \label{fig::sensity_appendix}
\vspace{-1.8em}
  \end{figure*}

  {
  \subsubsection{Visual Analysis}
  Due to spatial constraints, we conducted visual experiments on Gowalla and Yelp2018. As illustrated in Figure \ref{fig::tsne} and \ref{fig::yelp_tsne}, as the ratio of noise data increases, the item embeddings via SL begin to entangle. When the ratio reaches 40\%, the representation appears to resemble a uniform distribution, which confuses the relationships among items. In contrast, the embeddings via our proposed BSL still achieve decent group-wise separation. We believe that this is a superior and more meaningful embedding that reflects the characteristics of collaborative filtering, where similar items are grouped in the embedding space.
  }

{
\subsection{Sensitivity Study (RQ4)}
In this section, we conducted sensitivity analysis experiments on parameters in SL/BSL, thereby assisting everyone in providing some valuable insights for parameter configuration in the new dataset. Specifically, we initially conducted a validation of the model's dimensions. In the previous experiments, we fixed the representation dimension at 64. However, in this section, we attempted dimensions such as 128, 256, and 512. Additionally, we also examined the performance variations based on the $\tau_1 / \tau_2$. The details are as follows:
\subsubsection{Impact of embedding dimension}
Our analysis confirms that (\cf Figure \ref{fig::ndcg_DIM}), even with increased embedding sizes (\eg 512), the integration of SL/BSL consistently improves model performance. Notably, MF equipped with SL/BSL still achieves comparable performance with SOTA models under the setting of large embedding sizes. Additionally, despite the marginal performance gains of SL/BSL as the dimensions increase, it showcases its robustness towards the model and feature dimensions. In particular, the high accuracy at lower dimensions also makes it suitable for practical scenarios.
\subsubsection{Impact of temperature}
We conducted additional experiments to investigate the model's performance with varying ratios of $\tau_1/\tau_2$ by adjusting the hyper-parameter $\tau_1$. The results are illustrated in Figure \ref{fig::sensity_appendix}. When the value of $\tau_1$ is excessively large (\eg $\tau_1/\tau_2=2$), which implies an overly small robust radius according to our Corollary III.1, the model has poor robustness against positive noise. With $\tau_1$ decreasing, leading to increased robustness radius, both the model's robustness and performance boost. However, a further reduction of $\tau_1$ can be detrimental to performance, since an excessively large robust radius might amplify the risk of encountering an implausible worst-case distribution in DRO. 
}

\section{Related Work}
\label{sec_6}
In this section, we review the most related work from the following perspectives.
\subsection{Recommendation System}
The primary goal of the recommendation system is to predict potential interactions. Recent research has primarily focused on developing advanced model architectures, ranging from matrix factorization (MF-based) methods \cite{hu2008collaborative}, VAE-based methods \cite{liang2018variational, ma2019learning, steck2019embarrassingly}, to GNN-based methods \cite{he2020lightgcn, ying2018graph, wang2019neural}. For instance, MF-based models project the ID of users (or items) into embedding vectors to reconstruct historical interactions, while GNN-based models aggregate neighborhood information to enhance user/item representations.
For further details, readers can refer to the survey \cite{wu2022survey}.

\subsection{Robustness in Recommender System}

There has been increasing attention to enhancing the robustness of recommendation systems. In the early stages, Michael et al. \cite{o2004efficient} improved robustness by incorporating neighbor selection into MF via an adversarial network \cite{he2018adversarial}. To combine with user-item graph, Chen et al. \cite{chen2021structured} and Feng et al. \cite{feng2021robust} designed several denoising strategies or modules, such as attention mechanism \cite{chen2021structured} or stochastic binary masks \cite{feng2021robust}. More recently, Wang et al. \cite{wang2021denoising} adaptively pruned noisy interactions during training by monitoring the training process. Furthermore, Lin et al. \cite{lin2023autodenoise} introduced AutoDenoise, serving as an agent in DRL to dynamically select noise-free and predictive data instances. Ge et al. \cite{ge2023automated} learned to automatically assign the most appropriate weights to each implicit data sample.
In general, the existing methods incur additional time and computational cost to achieve robustness, or rely on access to prior knowledge to guide the denoising process. 
In contrast, our work tackles the challenge of denoising from the perspective of innovating the loss function architecture.

\subsection{Loss functions in Recommender System}
The loss function is critical in recommendation systems. Various loss functions have been proposed to address different challenges in collaborative filtering (CF). \cite{mao2021simplex} proposed the cosine contrastive loss inspired by contrastive learning. To combat the negative impact of popularity bias in CF models, \cite{zhang2022incorporating} integrated bias-aware margins into contrastive loss. Auto-ML technique has been employed by \cite{zhao2021autoloss} to search for the optimal loss design. They developed a novel controller network that can dynamically adjust the loss probabilities in a differentiable manner. Although \cite{mao2021simplex} and \cite{zhang2022incorporating} can be seen as improvements to SL, they lack theoretical analysis of their advantages or suffer from the efficiency issues caused by additional optimization.

To the best of our knowledge, there is limited work studying the theoretical properties of SL in RS. For example, \cite{bruch2019analysis} related SL to NDCG and proved SL is bounded by NDCG under multiple relaxations. However, their analysis struggles to explain the robustness, role of $\tau$, and fairer recommendation. \cite{wu2022effectiveness} attributed the fairness arising from SL to the popularity-aware sampling.
Nevertheless, we observe that SL maintains fairness and strong performance even with uniform sampling. 

By examining the Softmax loss from other areas, we identified a recent study \cite{wu2023ADNCE} that analyzes contrastive loss from DRO perspective, which is closely related to our research. The distinctions between our work with \cite{wu2023ADNCE} are threefold: 1) Our theoretical and empirical analyses of SL are specifically tailored to the recommendation domain; 2) We demonstrate both the robustness and fairness that are inherent in the SL when applied to recommendation systems; 3) We have developed a novel loss function BSL, which is specifically designed for recommender systems.


\subsection{Temperature in Softmax Loss}
Limited work has analyzed temperature in SL, primarily from the community of contrastive learning (CL). 
For instance, Wang et al. \cite{wang2021understanding, wu2022effectiveness} attributed the success of CL to its inherent hardness-awareness, while Zhang et al. \cite{zhang2022does} focused on the decorrelation ability. Additionally, Zhang et al. \cite{zhang2022dual} proposed dual temperature in the vector and scalar components, which is different from our Log-Expectation-Exp structure. Furthermore, \cite{kukleva2023temperature} demonstrated that a simple cosine schedule for achieving dynamic $\tau$ can significantly improve long-tail performance.
However, these studies regard temperature as a heuristic design. In contrast, our work provides novel theoretical insights into the importance of temperature for robustness in SL.

\section{Conclusion and Future Work}
\label{sec_7}
In this work, we focus on explaining the essence of softmax loss in recommendation system. 
We prove that: 1) Optimizing SL is equivalent to performing DRO on the negative parts in vanilla pointwise loss, giving it the ability to alleviate negative sample noise; {2) Comparing with other loss functions, SL implicitly penalizes the prediction variance, yielding fairer results;} Based on the above analysis, we propose BSL that mirrors the advantageous structure of the negative loss part to the positive part such that  BSL enjoys the merit of both positive and negative denoising. 
In our experiments, we evaluate BSL on numerous backbones to validate the effectiveness of our model and demonstrate its superiority.


This work represents the first step towards bridging the gap between SL and its superiority and robustness.
We believe that the proposed DRO method, which is generally used to deal with noisy data in RS, holds promising potential for extending to other fields of denoising. Additionally, exploring the role of different loss functions in fairness from our perspective is a potential direction for future research.

\section*{Acknowledgment}
This research is supported by the National Natural Science Foundation of China (9227010114, 62302321), the University Synergy Innovation Program of Anhui Province (GXXT-2022-040) and supported by the advanced computing resources provided by the Supercomputing Center of Hangzhou City University.

\appendix
\section{Detailed implementation}
\label{se:app}
In this section, we present the implementation details of our BSL in different backbones. During training, we use cosine similarity and BSL to calculate the training loss. During testing, we use cosine similarity for prediction score in MF and inner product in GCN-based models (NGCF, LightGCN). In "Negative Sampling," we use efficient uniform sampling techniques, while in "In Batch," we treat other positive items in a batch as negative samples.
\begin{table}[h]
  \caption{The difference between MF, NGCF and LightGCN.}
  \label{tab:proof_}
  \centering
  \begin{tabular}{c|ccc}
      \toprule
      Backbone & Training & Testing & Sampling Method\\
      \hline
       MF & cosine similarity & cosine similarity & Negative Sampling \\
       NGCF & cosine similarity & inner product  & In Batch\\
       LightGCN & cosine similarity & inner product & In Batch\\
      \bottomrule
  \end{tabular}
\end{table}
\vspace{-1.5em}

\begin{algorithm}[h]
  \caption{BSL Pseudocode (Negative Sampling)}
  \label{alg:code}
  \definecolor{codeblue}{rgb}{0.25,0.5,0.5}
  \definecolor{codekw}{rgb}{0.85, 0.18, 0.50}
  \lstset{
    backgroundcolor=\color{white},
    basicstyle=\fontsize{7.5pt}{7.5pt}\ttfamily\selectfont,
    columns=fullflexible,
    breaklines=true,
    captionpos=b,
    commentstyle=\fontsize{7.5pt}{7.5pt}\color{codeblue},
    keywordstyle=\fontsize{7.5pt}{7.5pt}\color{codekw},
  }
  \begin{lstlisting}[language=python, mathescape]
  # f: user and item embedding table
  # t1: temperature scaling for positive samples
  # t2: temperature scaling for negative samples
  for (u,i,j) in loader: 
  # load a minibatch (u,i) with m negative samples
  # dimension u: [B]; i:[B]; j:[B,m]
    emb_u, emb_i, emb_j = f(u), f(i), f(j)
    # dimension u: [B, D]; i:[B, D]; j:[B, m, D]
    L = loss_fn(emb_u, emb_i, emb_j)
    L.backward() # back-propagate
    update(f)    # Adam update

  def loss_fn(emb_u, emb_i, emb_j):
    emb_u = normalize(emb_u, dim=1) # l2-normalize
    emb_i = normalize(emb_u, dim=1) # l2-normalize
    emb_j = normalize(emb_u, dim=1) # l2-normalize
    pos_score=(emb_u*emb_i).sum(dim=1) # dimension: [B]
    neg_score=(emb_u.unsqueeze(1)*emb_j).sum(dim=2) 
    # dimension: [B, m]
    L(BSL) = -((pos_score/t1).exp() / (neg_score/t2).exp().sum().pow(t1/t2) ).log()
    return L(BSL)
  \end{lstlisting}
  \end{algorithm}
  
\vspace{-1.5em}

  \begin{algorithm}[h]
    \caption{BSL Pseudocode (In Batch), PyTorch-like}
    \label{alg:code2}
    \definecolor{codeblue}{rgb}{0.25,0.5,0.5}
    \definecolor{codekw}{rgb}{0.85, 0.18, 0.50}
    \lstset{
      backgroundcolor=\color{white},
      basicstyle=\fontsize{7.5pt}{7.5pt}\ttfamily\selectfont,
      columns=fullflexible,
      breaklines=true,
      captionpos=b,
      commentstyle=\fontsize{7.5pt}{7.5pt}\color{codeblue},
      keywordstyle=\fontsize{7.5pt}{7.5pt}\color{codekw},
    }
    \begin{lstlisting}[language=python, mathescape]
    # f: user and item embedding table
    # t1: temperature scaling for positive samples
    # t2: temperature scaling for negative samples
    for (u,i) in loader: 
    # load a minibatch (u,i)
    # dimension u: [B]; i:[B]
      emb_u, emb_i = f(u), f(i)
      # dimension u: [B, D]; i:[B, D]; j:[B, m, D]
      L = loss_fn(emb_u, emb_i)
      L.backward() # back-propagate
      update(f)    # Adam update
  
    def loss_fn(emb_u, emb_i):
      emb_u = normalize(emb_u, dim=1) # l2-normalize
      emb_i = normalize(emb_u, dim=1) # l2-normalize

      s=emb_u.matmul(emb_i.t()) # pairwise Similarity

      B = emb_u.shape(0)
      mask_index = mask_index(B)
      pos_score=s[~mask_index] # dimension: [B]
      neg_score=s[mask_index].view(B, -1)

      L(BSL) = -((pos_score/t1).exp() / (neg_score/t2).exp().sum().pow(t1/t2) ).log()
      return L(BSL)
    
    def mask_index(batch_size):
      negative_mask = ones((batch_size, batch_size))
      for i in range(batch_size):
          negative_mask[i, i] = 0
    return negative_mask
    \end{lstlisting}
    \end{algorithm}

\newpage

\bibliographystyle{unsrtnat}
\bibliography{sample-base}
\clearpage

\noindent\hrulefill\ \textbf{Summary of Updates}\ \hrulefill

We have made substantial revisions to the paper based on the comments provided by the meta-reviewer and three reviewers. The major changes are as follows:
\begin{enumerate}
  \item We have considered more SOTA models including SGL (SIGIR'21) \cite{wu2021self}, SimSGL (SIGIR'22) \cite{yu2022graph}  and LightGCL (ICLR'23) \cite{caisimple} and conducted experiments on more metrics including NDCG@5, NDCG@10, NDCG@15. The results are presented in Table \ref{overall_res_appendix} and Figure \ref{fig::ndcg_5_10_15}.
  \item We have conducted additional sensitivity experiments to investigate the model's performance with varying ratios of $\tau_1/\tau_2$ and embedding dimension. The results are illustrated in Figure \ref{fig::sensity_appendix} and Figure \ref{fig::ndcg_DIM}.
  \item We have added some statements to emphasize the significance of fairness stemming from DRO. Furthermore, additional analyses and experiments have been conducted in Section \ref{sec_3} to demonstrate the fairness of SL.
  \item We have carefully checked the entire paper, polishing the language and figures, as well as typos, confusing expressions, and other minor issues in the revised version.
\end{enumerate}
We thank all reviewers for the careful reading and valuable feedback. We will address their concerns as follows.

\noindent\hrulefill\ \textbf{Response to Reviewer \#8}\ \hrulefill

We greatly appreciate your acknowledgement of our contributions and your insightful comments. In what follows, we provide responses to the concerns you have raised:

\Review{1}{The author should survey more SOTA models and evaluate the results are true.}

\Response{1}{Thank you for the insightful suggestion. We incorporated additional experiments on more SOTA models including  SGL (SIGIR'21) \cite{wu2021self}, SimSGL (SIGIR'22) \cite{yu2022graph}  and LightGCL (ICLR'23) \cite{caisimple}. The results are presented in Table \ref{overall_res_appendix}. Here we applied both SL and BSL losses to those SOTA methods and observed that both losses enhance model performance, with BSL yielding more significant improvements. These results are in alignment with our analyses and demonstrate the effectiveness of the proposed loss. }



\Review{2}{The proposed BSL is a slightly enhanced SL. The novelty part of this paper is somewhat not sufficient. Though the claim of SL/BSL is good, more justifications and recommendation models should be added.}

\Response{2}{
Thank you for raising these concerns. First and foremost, we wish to highlight the novelty of our work, which lies not only in proposing a novel loss function but also in conducting a comprehensive analysis of the SL loss's properties. This includes theoretically demonstrating the robustness and fairness of SL and elucidating the role of temperature. We contend that such analyses are crucial, as they provide insights that could enable researchers in this field to deepen their understanding of loss properties and potentially inspire further research along this line. The introduction of our BSL is, in fact, a direct outcome of these analyses. Although BSL is a simple revision of SL, it has been shown to yield significant enhancements in performance across a variety of backbone models and metrics.

Furthermore, we are very grateful for the reviewers' constructive suggestions and have accordingly incorporated additional state-of-the-art recommendation models, metrics, and visualization analysis into our experiments. For further details, the reviewer is directed to our Response 1, 3, 4.

}

\Review{3}{How about other metric settings? For example, NDCG@5, NDCG@10, NDCG@15.}

\Response{3}{Thank you for the valuable suggestion. We have included additional experiments with various metrics ( \cf Figure \ref{fig::ndcg_5_10_15}). Our analysis confirms that, across the metrics of NDCG@5, NDCG@10, and NDCG@15, the integration of BSL consistently enhances model performance. This enhancement is substantial enough to allow basic models (\eg MF, lightGCN) to surpass SOTA models (\eg SimSGL, NCL).  }

\Review{4}{Figure 8 shows the visualization of Gowalla by adding some noise data. How about other datasets ?}

\Response{4}{Thank you for the valuable suggestion. We have conducted visualization experiments on other three datasets and drawn similar conclusions as on Gowalla. However, given the constraints of paper length, we simply report the visualization results from two representative datasets, as illustrated in Figures \ref{fig::tsne} and \ref{fig::yelp_tsne}}.

\noindent\hrulefill\ \textbf{Response to Reviewer \#10}\ \hrulefill

We sincerely appreciate your recognition of our work and deeply regret any confusion caused by some presentation issues within our paper. Your detailed comments are highly valued. In the revised version, we have refined the paper in accordance with your feedback, rectifying any typographical errors or presentation issues. In the following, we provide responses to the questions you have raised:

\Review{1}{Do the findings still hold with a larger embedding size?}

\Response{1}{Thank you for the valuable suggestion. In the revised manuscript, we have incorporated further experiments with diverse embedding dimensions, as detailed in Figure \ref{fig::ndcg_DIM}. Our analysis confirms that, even with increased embedding sizes (\eg 512), the integration of SL/BSL consistently improves model performance. Notably, MF equipped with SL/BSL still achieves comparable performance with SOTA models under the setting of large embedding sizes.}

\Review{2}{How about the sensitivity of BSL \wrt the hyper-parameter $\tau_1$?}

\Response{2}{Thanks for the insightful suggestion. In the revised manuscript, we have conducted additional experiments to investigate the model's performance with varying ratios of $\tau_1/\tau_2$ by adjusting the hyper-parameter $\tau_1$. The results are illustrated in Figure \ref{fig::sensity_appendix}. When the value of $\tau_1$ is excessively large (\eg $\tau_1/\tau_2=2$), which implies an overly small robust radius according to our Corollary \ref{coro1}, the model has poor robustness against positive noise. As $\tau_1$ decreases (\ie increasing robustness radius), both the model's robustness and performance are boosted. However, a further reduction of $\tau_1$ can be detrimental to performance, since an excessively large robust radius might amplify the risk of encountering an implausible worst-case distribution in DRO. }
\Review{3}{ Why Table 1 does not show the results of SimpleX and NIA-GCN on the Amazon dataset ?}

\Response{3}{We apologize for the inconvenience, but we encountered difficulties in replicating the results of SimpleX and NIA-GCN on the Amazon dataset. On the one hand, their original papers do not disclose performance on the Amazon dataset. On the other hand, despite our diligent efforts to fine-tune a substantial number of hyperparameters, we were unable to obtain decent results of these two methods. }


\noindent\hrulefill\ \textbf{Response to Reviewer \#14}\ \hrulefill

We appreciate your recognition of our contribution to the connection between SL and DRO. We also express our gratitude for your insightful inquiries regarding BSL. Below, we present responses to your comments:

\Review{1}{The fairness proportion is also relatively small in the paper. Besides, there is still a minor lack of clarity on how the fairness theory (Equation \eqref{lam1_eq}) can be used to support empirical results (Figure \ref{fig::t_changes_performance}a).}

\Response{1}{Thank you for the constructive comments to enhance our paper. In the revised manuscript,  we have placed greater emphasis on the aspect of fairness within both the Contributions and Conclusions. Furthermore, we provide more analyses on the fairness of SL in Section \ref{sec_3}.  

\textbf{Meanwhile, the connection between Equation \eqref{lam1_eq} and Figure \ref{fig::t_changes_performance}a stems from the fundamental role of the negative sample variance penalty.} Recent research \cite{williamson2019fairness} has indicated that the variance penalty corresponds to fairness. From Equation \eqref{lam1_eq}, we can conclude that SL implicitly introduces a regularizer that penalizes the variance of the model predictions on the negative instances. It is important to note that in typical recommender systems, recommendation model is prone to give extensive higher scores on popular items than unpopular ones, incurring popularity bias \cite{chen2021autodebias}. SL could mitigate this effect to a certain degree, as it inherently reduces the variance of model predictions, which implies that the predictions are likely to be closer. Consequently, the discrepancy in model predictions between popular and unpopular items is diminished, leading to more fair recommendation results. Unpopular items get more opportunity in recommendation, and thus we observe better NDCG@20 for those unpopular groups (\eg groups 1-5).



To establish a clear contrast, we conducted ablation studies by removing the variance penalty term and comparing the performance of the standard SL loss with the revised loss. The results are presented in Figure \ref{fig::variance_w_wo}. As can be seen, the exclusion of the variance term significantly exacerbates the unfairness of recommendation --- \ie better performance on popular groups (\eg groups 8-10) while much worse performance on unpopular groups (\eg groups 1-5).

\begin{figure}[h]\centering
  \includegraphics[width=0.5\linewidth]{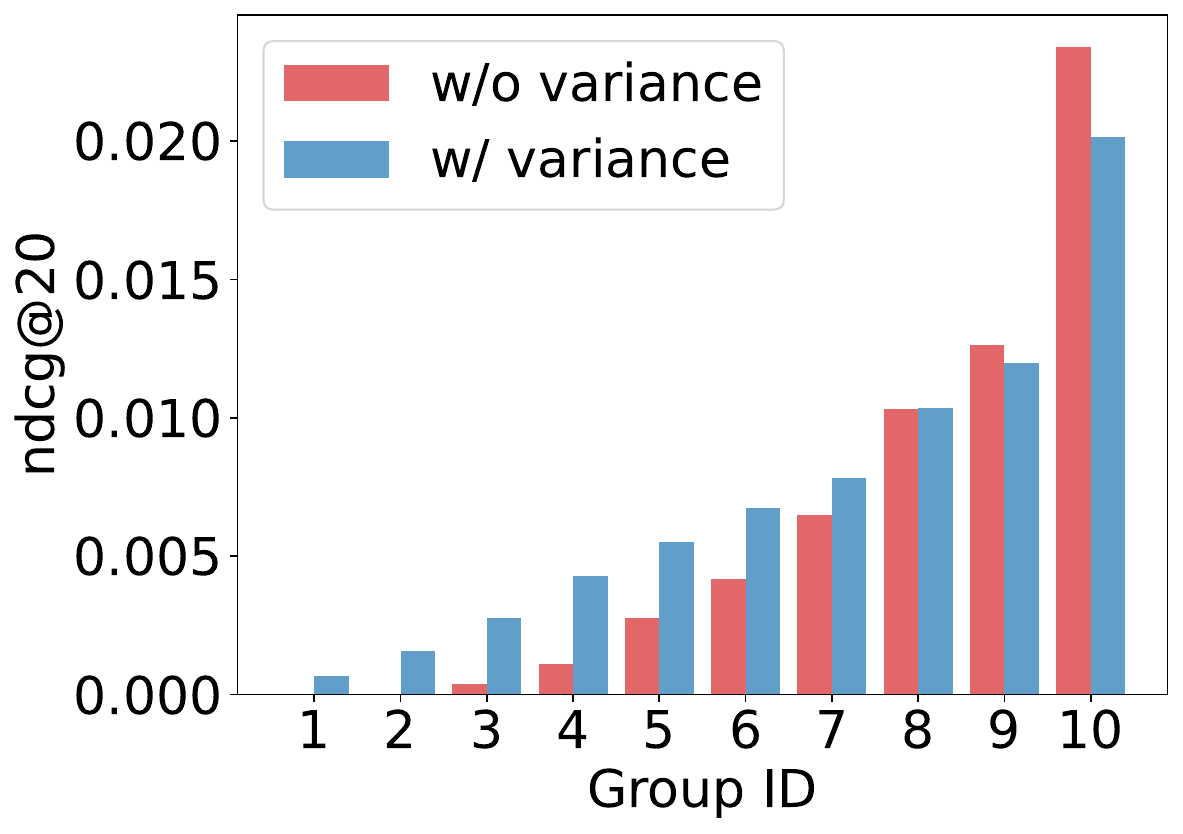}
  \caption{Performance comparison \wrt different item groups. The item groups with larger ID suggest these items have larger popularity.}
  \label{fig::variance_w_wo}
  \vspace{-1.8em}
\end{figure}
}
\Review{2}{1. P* in Figure \ref{fig::ppt_Performance} refers to true distribution but in paragraph it refers to worst case distribution. 2. The term ``potential distribution'' in Figure \ref{fig::ppt_Performance} does not appear anywhere else, does it  mean the worst distribution mentioned before?}

\Response{2}{Thank you for identifying these problems. In the revised manuscript, we polished Figure \ref{fig::ppt_Performance} and denoted the ideal distribution as $p^{ideal}$ and revised the term ``potential distribution'' to the ``worst-case distribution''. 
}



\Review{3}{Why SL performs much worse than BPR, MSE in Figure \ref{fig::t_changes_performance}a ?}

\Response{3}{ We apologize for the confusion arising from the previously unclear description. In fact, Figure \ref{fig::t_changes_performance}a demonstrate SL exhibits better fairness compared with other losses. Here we classify items into 10 groups according to item popularity and test recommendation performance for each item group \cite{zhang2021causal}. A higher Group ID indicates greater item popularity. As depicted in Figure \ref{fig::t_changes_performance}a, the SL loss significantly improves model performance for less popular item groups (groups 1-7), albeit at the cost of reduced performance for more popular item (groups 8-10). This result demonstrates SL achieves more fair recommendation, where the performance gap between popular items and unpopular items has been narrowed. Furthermore, SL achieves overall better performance (as shown in Table \ref{overall_res}) due to the improvements over the majority of groups. In the revised manuscript, we have provided more detailed explanations in Section \ref{sec_3} to facilitate a clearer understanding for readers.


}



\Review{4}{In figures \ref{fig::pos_ratio} and \ref{fig::num_neg}, SL keeps performing bad in Gowalla dataset. The author should explain more than just doubting there are positive noise.}

\Response{4}{We appreciate your insightful comments. We surmise
 that the observed phenomenon could be attributed to the characteristics of the Gowalla dataset, which may contain less noisy negative instances than other datasets. Consequently, the benefits of the SL method are less pronounced. Moreover, considering more complex function of SL than other losses, SL might be relatively more hard to optimize. This could explain the suboptimal performance of SL with a very small number of negative instances (\eg 32) in Figure \ref{fig::num_neg}. However, as the number of negative instances increases and the risk of sampling noisy data increases, SL can achieve comparable performance with the best compared losses (\eg SL performs similar to BCE with 2048 negative instances). 
}

\Review{5}{Can elaborate more on “the fundamental reasons for SL's effectiveness remain poorly explored” mentioned in Abstract.}

\Response{5}{ Thank you for the valuable suggestion. In the revised manuscript, we give better descriptions of recent work on SL in Abstract.}

\end{document}